\documentclass[twoside]{article}
\usepackage{amsthm,amsmath}
\usepackage{amssymb}
\usepackage{bm}
\usepackage{bbm}
\usepackage{url}
\usepackage{url}
\usepackage{graphicx}
\usepackage{color}
\usepackage{tikz}
\usepackage{graphicx}
\usepackage{caption}
\usepackage{float}

\usepackage[accepted]{aistats2015}
\begin{document}

% If your paper is accepted and the title of your paper is very long,
% the style will print as headings an error message. Use the following
% command to supply a shorter title of your paper so that it can be
% used as headings.
%
%\runningtitle{I use this title instead because the last one was very long}

% If your paper is accepted and the number of authors is large, the
% style will print as headings an error message. Use the following
% command to supply a shorter version of the authors names so that
% they can be used as headings (for example, use only the surnames)
%
%\runningauthor{Surname 1, Surname 2, Surname 3, ...., Surname n}

\newtheorem{definition}{Definition}[section]
\newtheorem{theorem}{Theorem}[section]
\newtheorem{lemma}{Lemma}[section]
\newtheorem{corollary}{Corollary}[section]

\def\layersep{2cm}
\def\layersepp{4cm}

\twocolumn[

\aistatstitle{The Loss Surfaces of Multilayer Networks}

\aistatsauthor{ Anna Choromanska\\ \texttt{achoroma@cims.nyu.edu} \And Mikael Henaff\\ \texttt{mbh305@nyu.edu} \And Michael Mathieu\\ \texttt{mathieu@cs.nyu.edu} \And G\'{e}rard {Ben Arous}\\ \texttt{benarous@cims.nyu.edu} \And Yann LeCun\\ \texttt{yann@cs.nyu.edu}}

\aistatsaddress{ \\Courant Institute of Mathematical Sciences\\ New York, NY, USA } ]

\begin{abstract}
We study the connection between the highly non-convex loss function of a simple model of the fully-connected feed-forward neural network and the Hamiltonian of the spherical spin-glass model under the assumptions of: i) variable independence, ii) redundancy in network parametrization, and iii) uniformity. These assumptions enable us to explain the complexity of the fully decoupled neural network through the prism of the results from random matrix theory. We show that for large-size decoupled networks the lowest critical values of the random loss function form a layered structure and they are located in a well-defined band lower-bounded by the global minimum. The number of local minima outside that band diminishes exponentially with the size of the network. We empirically verify that the mathematical model exhibits similar behavior as the computer simulations, despite the presence of high dependencies in real networks. We conjecture that both simulated annealing and SGD converge to the band of low critical points, and that all critical points found there are local minima of high quality measured by the test error. This emphasizes a major difference between large- and small-size networks where for the latter poor quality local minima have non-zero probability of being recovered. Finally, we prove that recovering the global minimum becomes harder as the network size increases and that it is in practice irrelevant as global minimum often leads to overfitting. 
\end{abstract}

\section{Introduction}

Deep learning methods have enjoyed a resurgence of interest in the last few years for such applications as image recognition~\cite{NIPS2012_4824}, speech recognition~\cite{hinton12}, and natural language processing~\cite{DBLP:conf/emnlp/WestonCA14}. Some of the most popular methods use multi-stage architectures composed of alternated layers of linear transformations and max function. In a particularly popular version, the max functions are known as ReLUs (Rectified Linear Units) and compute the mapping $y = {\rm max}(x,0)$ in a pointwise fashion~\cite{DBLP:conf/icml/NairH10}. In other architectures, such as convolutional networks~\cite{lecun-gradientbased-learning-applied-1998} and maxout networks~\cite{Goodfellow_maxout_2013}, the max operation is performed over a small set of variable within a layer.

The vast majority of practical applications of deep learning use supervised learning with very deep networks. The supervised loss function, generally a cross-entropy or hinge loss, is minimized using some form of stochastic gradient descent (SGD)~\cite{bottou-98x}, in which the gradient is evaluated using the back-propagation procedure~\cite{lecun-98b}.

The general shape of the loss function is very poorly understood. In the early days of neural nets (late 1980s and early 1990s), many researchers and engineers were experimenting with relatively small networks, whose convergence tends to be unreliable, particularly when using batch optimization.  Multilayer neural nets earned a reputation of being finicky and unreliable, which in part caused the community to focus on simpler method with convex loss functions, such as kernel machines and boosting.

However, several researchers experimenting with larger networks and SGD had noticed that, while multilayer nets do have many local minima, the result of multiple experiments consistently give very similar performance. This suggests that, while local minima are numerous, they are relatively easy to find, and they are all more or less equivalent in terms of performance on the test set. The present paper attempts to explain this peculiar property through the use of random matrix theory applied to the analysis of critical points in high degree polynomials on the sphere.

We first establish that the loss function of a typical multilayer net with ReLUs can be expressed as a polynomial function of the weights in the network, whose degree is the number of layers, and whose number of monomials is the number of paths from inputs to output. As the weights (or the inputs) vary, some of the monomials are switched off and others become activated, leading to a piecewise, continuous polynomial whose monomials are switched in and out at the boundaries between pieces.

An important question concerns the distribution of critical points (maxima, minima, and saddle points) of such functions. Results from random matrix theory applied to spherical spin glasses have shown that these functions have a combinatorially large number of saddle points. Loss surfaces for large neural nets have many local minima that are essentially equivalent from the point of view of the test error, and these minima tend to be highly degenerate, with many eigenvalues of the Hessian near zero.

We empirically verify several hypotheses regarding learning with large-size networks:
%\vspace{-0.1in}
\begin{itemize}
\item For large-size networks, most local minima are equivalent and yield similar performance on a test set.
%\vspace{-0.1in}
\item The probability of finding a ``bad'' (high value) local minimum is non-zero for small-size networks and decreases quickly with network size.
%\vspace{-0.1in}
\item Struggling to find the global minimum on the training set (as opposed to one of the many good local ones) is not useful in practice and may lead to overfitting.
\end{itemize}
%\vspace{-0.1in}
The above hypotheses can be directly justified by our theoretical findings. We finally conclude the paper with brief discussion of our results and future research directions in Section~\ref{sec:ConandFutWork}. 

We confirm the intuition and empirical evidence expressed in previous works that the problem of training deep learning systems resides with avoiding saddle points and quickly ``breaking the symmetry'' by picking sides of saddle points and choosing a suitable attractor~\cite{lecun-98b,DBLP:journals/corr/SaxeMG13,DBLP:journals/corr/DauphinPGCGB14}.

What is new in this paper? To the best of our knowledge, this paper is the first work providing a theoretical description of the optimization paradigm with neural networks in the presence of large number of parameters. It has to be emphasized however that this connection relies on a number of possibly unrealistic assumptions. It is also an attempt to shed light on the puzzling behavior of modern deep learning systems when it comes to optimization and generalization.

\section{Prior work}
\label{sec:PriorWork}

In the 1990s, a number of researchers studied the convergence of gradient-based learning for multilayer networks using the methods of statistical physics, i.e.~\cite{saad1995exact}, and the edited works~\cite{saad2009line}. Recently, Saxe~\cite{DBLP:journals/corr/SaxeMG13} and Dauphin~\cite{DBLP:journals/corr/DauphinPGCGB14} explored the statistical properties of the error surface in multi-layer architectures, pointing out the importance of saddle points. 

Earlier theoretical analyses~\cite{Baldi:1989:NNP:70359.70362,wigner_semicircle,Fyodorov2007,Bray2007} suggest the existence of a certain structure of critical points of random Gaussian error functions on high dimensional continuous spaces. They imply that critical points whose error is much higher than the global minimum are exponentially likely to be saddle points with many negative and approximate plateau directions whereas all local minima are likely to have an error very close to that of the global minimum (these results are conveniently reviewed in~\cite{DBLP:journals/corr/DauphinPGCGB14}). The work of~\cite{DBLP:journals/corr/DauphinPGCGB14} establishes a strong empirical connection between neural networks and the theory of random Gaussian fields by providing experimental evidence that the cost function of neural networks exhibits the same properties as the Gaussian error functions on high dimensional continuous spaces. Nevertheless they provide no theoretical justification for the existence of this connection which instead we provide in this paper.

This work is inspired by the recent advances in random matrix theory and the work of~\cite{AAC2010} and~\cite{AAC2013}. The authors of these works provided an asymptotic evaluation of the complexity of the spherical spin-glass model (the spin-glass model originates from condensed matter physics where it is used to represent a magnet with irregularly aligned spins). They discovered and mathematically proved the existence of a layered structure of the low critical values for the model's Hamiltonian which in fact is a Gaussian process. Their results are not discussed in details here as it will be done in Section~\ref{sec:theory} in the context of neural networks. We build the bridge between their findings and neural networks and show that the objective function used by neural network is analogous to the Hamiltonian of the spin-glass model under the assumptions of: i) variable independence, ii) redundancy in network parametrization, and iii) uniformity, and thus their landscapes share the same properties. We emphasize that the connection between spin-glass models and neural networks was already explored back in the past (a summary can be found in~\cite{Dotsenko1995}). In example in~\cite{PhysRevA.32.1007} the authors showed that the long-term behavior of certain neural network models are governed by the statistical mechanism of infinite-range Ising spin-glass Hamiltonians. Another work~\cite{0305-4470-30-23-009} examined the nature of the spin-glass transition in the Hopfield neural network model. None of these works however make the attempt to explain the paradigm of optimizing the highly non-convex neural network objective function through the prism of spin-glass theory and thus in this respect our approach is very novel.

\section{Deep network and spin-glass model}
\label{sec:NNSG}

\subsection{Preliminaries}

For the theoretical analysis, we consider a simple model of the fully-connected feed-forward deep network with a single output and rectified linear units. We call the network $\mathcal{N}$. We focus on a binary classification task. Let $X$ be the random input vector of dimensionality $d$. Let $(H-1)$ denote the number of hidden layers in the network and we will refer to the input layer as the $0^{\text{th}}$ layer and to the output layer as the $H^{\text{th}}$ layer. Let $n_i$ denote the number of units in the $i^{\text{th}}$ layer (note that $n_0 = d$ and $n_H = 1$). Let $W_i$ be the matrix of weights between $(i - 1)^{\text{th}}$ and $i^{th}$ layers of the network. Also, let $\sigma$ denote the activation function that converts a unit's weighted input to its output activation. We consider linear rectifiers thus $\sigma(x) = \max(0,x)$. We can therefore write the (random) network output $Y$ as
\[Y = q\sigma(W_H^{\top}\sigma(W_{H-1}^{\top}\dots\sigma(W_1^{\top}X)))\dots),
\]
where $q = \sqrt{(n_0n_1...n_H)^{(H-1)/2H}}$ is simply a normalization factor. The same expression for the output of the network can be re-expressed in the following way:
\begin{equation}
Y = q\sum_{i=1}^{n_0}\sum_{j = 1}^\gamma X_{i,j}A_{i,j}\prod_{k = 1}^{H}w_{i,j}^{(k)},
\label{eq:befrein}
\end{equation}
where the first summation is over the network inputs and the second one is over all paths from a given network input to its output, where $\gamma$ is the total number of such paths (note that $\gamma = n_1n_2\dots n_H$). Also, for all $i = \{1,2,\dots,n_0\}$: $X_{i,1} = X_{i,2} = \dots = X_{i,\gamma}$. Furthermore, $w_{i,j}^{(k)}$ is the weight of the $k^{\text{th}}$ segment of path indexed with $(i,j)$ which connects layer $(k-1)$ with layer $k$ of the network. Note that each path corresponds to a certain set of $H$ weights, which we refer to as a \textit{configuration of weights}, which are multiplied by each other. Finally, $A_{i,j}$ denotes whether a path $(i,j)$ is active ($A_{i,j} = 1$) or not ($A_{i,j} = 0$). 

\begin{definition}
The mass of the network $\Psi$ is the total number of all paths between all network inputs and outputs: $\Psi = \prod_{i=0}^Hn_i$. Also let $\Lambda$ as $\Lambda = \sqrt[H]{\Psi}$.
\end{definition}

\begin{definition}
The size of the network $N$ is the total number of network parameters: $N = \sum_{i=0}^{H-1}n_in_{i+1}$.
\end{definition}

The mass and the size of the network depend on each other as captured in Theorem~\ref{thm:arge}. All proofs in this paper are deferred to the Supplementary material. 

\begin{theorem}
Let $\Psi$ be the mass of the network, $d$ be the number of network inputs and $H$ be the depth of the network. The size of the network is bounded as
\[\Psi^2H = \Lambda^{2H}H \geq N \geq \sqrt[H]{\Psi^2}\frac{H}{\sqrt[H]{d}} \geq \sqrt[H]{\Psi} = \Lambda.
\]
\label{thm:arge}
\end{theorem}
We assume the depth of the network $H$ is bounded. Therefore $N \rightarrow \infty$ iff $\Psi \rightarrow \infty$, and $N \rightarrow \infty$ iff $\Lambda \rightarrow \infty$. 

In the rest of this section we will be establishing a connection between the loss function of the neural network and the Hamiltonian of the spin-glass model. We next provide the outline of our approach. 

\subsection{Outline of the approach}

In Subsection~\ref{sec:Approximation} we introduce randomness to the model by assuming $X$'s and $A$'s are random. 
We make certain assumptions regarding the neural network model. 
First, we assume certain distributions and mutual dependencies concerning the random variables $X$'s and $A$'s.
We also introduce a spherical constraint on the model weights. We finally make two other assumptions regarding the redundancy of network parameters and their uniformity, both of which are justified by empirical evidence in the literature. 
These assumptions will allow us to show in Subsection~\ref{subsec:LFSG} that the loss function of the neural network, after re-indexing terms\footnote{The terms are re-indexed in Subsection~\ref{sec:Approximation} and it is done to preserve consistency with the notation in~\cite{AAC2010} where the proofs of the results of Section~\ref{sec:theory} can be found.}, has the form of a centered Gaussian process on the sphere $\mathcal{S} = S^{\Lambda-1}(\sqrt{\Lambda})$, which is equivalent to the Hamiltonian of the $H$-spin spherical spin-glass model, given as
\begin{equation}
\mathcal{L}_{\!\Lambda,H}({\bm {\tilde{w}}}) \!=\! \frac{1}{\Lambda^{\!(H\!-\!1)/2}}\!\!\!\!\!\!\!\sum_{i_1,i_2,\dots,i_H=1}^{\Lambda}\!\!\!\!\!\!\!\!\!\!\!X_{i_1,i_2,\dots,i_H}\!\tilde{w}_{i_1}\!\tilde{w}_{i_2}\!\!\dots\!\tilde{w}_{i_H},
\label{eq:spinglass}
\end{equation}
with spherical constraint
\begin{equation}
\frac{1}{\Lambda}\sum_{i=1}^{\Lambda}\tilde{w}_i^2 = 1.
\label{eq:spherical}
\end{equation}
The redundancy and uniformity assumptions will be explained in Subsection~\ref{sec:Approximation} in detail. However, on the high level of generality the redundancy assumption enables us to skip superscript $(k)$ appearing next to the weights in Equation~\ref{eq:befrein} (note it does not appear next to the weights in Equation~\ref{eq:spinglass}) by determining a set of unique network weights of size no larger than $N$, and the uniformity assumption ensures that all ordered products of unique weights appear in Equation~\ref{eq:spinglass} the same number of times. 

An asymptotic evaluation of the complexity of $H$-spin spherical spin-glass models via random matrix theory was studied in the literature~\cite{AAC2010} where a precise description of the energy landscape for the Hamiltonians of these models is provided. In this paper (Section~\ref{sec:theory}) we use these results to explain the optimization problem in neural networks. 

\subsection{Approximation}
\label{sec:Approximation}

\paragraph{Input}
We assume each input $X_{i,j}$ is a normal random variable such that $X_{i,j}\sim N(0,1)$. 
Clearly the model contains several dependencies as one input is associated with many paths in the network. 
That poses a major theoretical problem in analyzing these models as it is unclear how to account for these dependencies. 
In this paper we instead study fully decoupled model~\cite{opac-b1095246}, where $X_{i,j}$'s are assumed to be independent. 
We allow this simplification as to the best of our knowledge there exists no theoretical description of the optimization paradigm with neural networks in the literature either under independence assumption or when the dependencies are allowed. 
Also note that statistical learning theory heavily relies on this assumption~\cite{hastie01statisticallearning} even when the model under consideration is much simpler than a neural network. Under the independence assumption we will demonstrate the similarity of this model to the spin-glass model. We emphasize that despite the presence of high dependencies in real neural networks, both models exhibit high similarities as will be empirically demonstrated. 

\paragraph{Paths}
We assume each path in Equation~\ref{eq:befrein} is equally likely to be active thus $A_{i,j}$'s will be modeled as Bernoulli random variables with the same probability of success $\rho$. By assuming the independence of $X$'s and $A$'s we get the following
\begin{equation}
\mathbb{E}_A[Y] = q\sum_{i=1}^{n_0}\sum_{j = 1}^\gamma X_{i,j}\rho\prod_{k = 1}^{H}w_{i,j}^{(k)}.
\label{eq:befrein2}
\end{equation}

\paragraph{Redundancy in network parametrization}
Let $\mathcal{W} = \{w_1,w_2,\dots,w_N\}$ be the set of all weights of the network. Let $\mathcal{A}$ denote the set of all $H$-length configurations of weights chosen from $\mathcal{W}$ (order of the weights in a configuration does matter). Note that the size of $\mathcal{A}$ is therefore $N^H$. Also let $\mathcal{B}$ be a set such that each element corresponds to the single configuration of weights from Equation~\ref{eq:befrein2}, thus $\mathcal{B} = \{(w_{1,1}^1,w_{1,1}^2,\dots,w_{1,1}^H),(w_{1,2}^1,w_{1,2}^2,\dots,w_{1,2}^H),\dots,\\(w_{n_0,\gamma}^1,w_{n_0,\gamma}^2,\dots,w_{n_0,\gamma}^H)\}$, where every single weight comes from set $\mathcal{W}$ (note that $\mathcal{B} \subset \mathcal{A}$). Thus Equation~\ref{eq:befrein2} can be equivalently written as 
\begin{equation}
Y_N \!:=\! \mathbb{E}_A[Y] \!=\! q\!\!\!\!\!\!\!\sum_{i_1,i_2,\dots,i_H=1}^{N}\!\!\!\!\!\!\!\sum_{j = 1}^{r_{i_1\!,i_2,\dots,i_H}}\!\!\!\!\!X_{i_1,i_2,\dots,i_H}^{(j)}\rho\!\prod_{k = 1}^{H}\!\!w_{i_k}.
\label{eq:befapprox}
\end{equation}
We will now explain the notation. It is over-complicated for purpose, as this notation will be useful later on. $r_{i_1,i_2,\dots,i_H}$ denotes whether the configuration $(w_{i_1},w_{i_2},\dots,w_{i_H})$ appears in Equation~\ref{eq:befrein2} or not, thus $r_{i_1,i_2,\dots,i_H} \in \{0\cup{1}\}$, and $\{X_{i_1,i_2,\dots,i_H}^{(j)}\}_{j=1}^{r_{i_1,i_2,\dots,i_H}}$ denote a set of random variables corresponding to the same weight configuration (since $r_{i_1,i_2,\dots,i_H} \in \{0\cup{1}\}$ this set has at most one element). Also $r_{i_1,i_2,\dots,i_H} = 0$ implies that summand $X_{i_1,i_2,\dots,i_H}^{(j)}\rho\prod_{k = 1}^{H}w_{i_k}$ is zeroed out). Furthermore, the following condition has to be satisfied: $\sum_{i_1,i_2,\dots,i_H=1}^{N}r_{i_1,i_2,\dots,i_H} = \Psi$. In the notation $Y_N$, index $N$ refers to the number of unique weights of a network (this notation will also be helpful later).  

Consider a family of networks which have the same graph of connections as network $\mathcal{N}$ but different edge weighting such that they only have $s$ unique weights and $s \leq N$ (by notation analogy the expected output of this network will be called $Y_s$). It was recently shown~\cite{NIPS2013_5025,DBLP:journals/corr/DentonZBLF14} that for large-size networks large number of network parameters (according to~\cite{NIPS2013_5025} even up to $95\%$) are redundant and can either be learned from a very small set of unique parameters or even not learned at all with almost no loss in prediction accuracy.  

\begin{definition}
A network $\mathcal{M}$ which has the same graph of connections as $\mathcal{N}$ and $s$ unique weights satisfying $s \leq N$ is called a $(s,\epsilon)$-\textit{reduction image} of $\mathcal{N}$ for some $\epsilon \in [0,1]$ if the prediction accuracy of $\mathcal{N}$ and $\mathcal{M}$ differ by no more than $\epsilon$ (thus they classify at most $\epsilon$ fraction of data points differently).
\end{definition}

\begin{theorem}
Let $\mathcal{N}$ be a neural network giving the output whose expectation $Y_N$ is given in Equation~\ref{eq:befapprox}. Let $\mathcal{M}$ be its $(s,\epsilon)$-reduction image for some $s \leq N$ and $\epsilon \in [0,0.5]$. By analogy, let $Y_s$ be the expected output of network $\mathcal{M}$. Then the following holds
\[corr(sign(Y_s),sign(Y_N)) \geq \frac{1-2\epsilon}{1+2\epsilon},
\]
where $corr$ denotes the correlation defined as $corr(A,B) = \frac{\mathbb{E}[(A - \mathbb{E}[A]])(B - \mathbb{E}[B]])}{std(A)std(B)}$, $std$ is the standard deviation and $sign(\cdot)$ denotes the sign of prediction ($sign(Y_s)$ and $sign(Y_N)$ are both random).
\label{thm:redun}
\end{theorem}

The redundancy assumption implies that one can preserve $\epsilon$ to be close to $0$ even with $s << N$.

\paragraph{Uniformity}
Consider the network $\mathcal{M}$ to be a $(s,\epsilon)$-reduction image of $\mathcal{N}$ for some $s \leq N$ and $\epsilon \in [0,1]$. The output $Y_s$ of the image network can in general be expressed as
\[Y_s = q\sum_{i_1,\dots,i_H=1}^{s}\sum_{j=1}^{t_{i_1,\dots,i_H}}X_{i_1,\dots,i_H}^{(j)}\rho\prod_{k = 1}^{H}w_{i_k},
\]
where $t_{i_1,\dots,i_H} \in \{\mathbb{Z}^{+}\cup{0}\}$ is the number of times each configuration $(w_{i_1},w_{i_2},\dots,w_{i_H})$ repeats in Equation~\ref{eq:befapprox} and $\sum_{i_1,\dots,i_H=1}^{s}t_{i_1,\dots,i_H} = \Psi$. We assume that unique weights are close to being evenly distributed on the graph of connections of network $\mathcal{M}$. We call this assumption a \textit{uniformity assumption}. Thus this assumption implies that for all $(i_1,i_2,\dots,i_H):i_1,i_2,\dots,i_H \in \{1,2,\dots,s\}$ there exists a positive constant $c \geq 1$ such that the following holds
\begin{equation}
\frac{1}{c}\cdot\frac{\Psi}{s^H} \leq t_{i_1,i_2,\dots,i_H} \leq c\cdot\frac{\Psi}{s^H}.
\label{eq:uniform}
\end{equation}
The factor $\frac{\Psi}{s^H}$ comes from the fact that for the network where every weight is uniformly distributed on the graph of connections (thus with high probability every node is adjacent to an edge with any of the unique weights) it holds that $t_{i_1,i_2,\dots,i_H} = \frac{\Psi}{s^H}$. For simplicity assume $\frac{\Psi}{s^H} \in \mathbb{Z}^{+}$ and $\sqrt[H]{\Psi} \in \mathbb{Z}^{+}$. Consider therefore an expression as follows
\begin{equation}
\hat{Y}_s = q\sum_{i_1,\dots,i_H=1}^{s}\sum_{j=1}^{\frac{\Psi}{s^H}}X_{i_1,\dots,i_H}^{(j)}\rho\prod_{k = 1}^{H}w_{i_k},
\label{eq:approx}
\end{equation}
which corresponds to a network for which the lower-bound and upper-bound in Equation~\ref{eq:uniform} match. Note that one can combine both summations in Equation~\ref{eq:approx} and re-index its terms to obtain
\begin{equation}
\hat{Y} := \hat{Y}_{(s= \Lambda)} = q\sum_{i_1,\dots,i_H=1}^{\Lambda}X_{i_1,\dots,i_H}\rho\prod_{k = 1}^{H}w_{i_k}.
\label{eq:approxfinal}
\end{equation}
The following theorem (Theorem~\ref{thm:unif}) captures the connection between $\hat{Y}_s$ and $Y_{s}$.

\begin{theorem}
Under the uniformity assumption of Equation~\ref{eq:uniform}, random variable $\hat{Y}_s$ in Equation~\ref{eq:approx} and random variable $Y_s$ in Equation~\ref{eq:befapprox} satisfy the following: $corr(\hat{Y}_s,Y_s) \geq \frac{1}{c^2}$.
\label{thm:unif}
\end{theorem}

\paragraph{Spherical constraint}
We finally assume that for some positive constant $\mathcal{C}$ weights satisfy the spherical condition 
\begin{equation}
\frac{1}{\Lambda}\sum_{i=1}^{\Lambda}w_i^2 = \mathcal{C}.
\label{eq:befspherical}
\end{equation}

Next we will consider two frequently used loss functions, absolute loss and hinge loss, where we approximate $Y_N$ (recall $Y_N := \mathbb{E}_A[Y]$) with $\hat{Y}$. 

\subsection{Loss function as a $H$-spin spherical spin-glass model}
\label{subsec:LFSG}

Let $\mathcal{L}^a_{\Lambda,H}(w)$ and $\mathcal{L}^h_{\Lambda,H}(w)$ be the (random) absolute loss and (random) hinge loss that we define as follows
\[\mathcal{L}^a_{\Lambda,H}({\bm w}) = \mathbb{E}_A[\:|Y_t - Y|\:]
\]
and
\[\mathcal{L}^h_{\Lambda,H}({\bm w}) = \mathbb{E}_A[\max(0,1-Y_tY)],
\]
where $Y_t$ is a random variable corresponding to the true data labeling that takes values $-S$ or $S$ in case of the absolute loss, where $S = \sup_{\bf{w}}{\hat{Y}}$, and $-1$ or $1$ in case of the hinge loss. Also note that in the case of the hinge loss $\max$ operator can be modeled as Bernoulli random variable, which we assume is \textit{independent} of $\hat{Y}$. Given that one can show that after approximating $\mathbb{E}_A[Y]$ with $\hat{Y}$ both losses can be generalized to the following expression
\[\mathcal{L}_{\!\Lambda,H}({\bm {\tilde{w}}}) = \mathcal{C}_1 + \mathcal{C}_2q\sum_{i_1,i_2,\dots,i_H=1}^{\Lambda}X_{i_1,i_2,\dots,i_H}\prod_{k = 1}^{H}\tilde{w}_{i_k},
\]
and $\mathcal{C}_1$, $\mathcal{C}_2$ are some constants and weights $\tilde{w}$ are simply scaled weights $w$ satisfying $\frac{1}{\Lambda}\sum_{i=1}^{\Lambda}\tilde{w}_i^2 = 1$. In case of the absolute loss the term $Y_t$ is incorporated into the term $\mathcal{C}_1$, and in case of the hinge loss it vanishes (note that $\hat{Y}$ is a symmetric random quantity thus multiplying it by $Y_t$ does not change its distribution). We skip the technical details showing this equivalence, and defer them to the Supplementary material. Note that after simplifying the notation by i) dropping the letter accents and simply denoting $\tilde{w}$ as $w$, ii) skipping constants $\mathcal{C}_1$ and $\mathcal{C}_2$ which do not matter when minimizing the loss function, and iii) substituting $q = \frac{1}{\Psi^{(H-1)/2H}} = \frac{1}{\Lambda^{(H-1)/2}}$, we obtain the Hamiltonian of the $H$-spin spherical spin-glass model of Equation~\ref{eq:spinglass} with spherical constraint captured in Equation~\ref{eq:spherical}. 

\section{Theoretical results}
\label{sec:theory}

In this section we use the results of the theoretical analysis of the complexity of spherical spin-glass models of~\cite{AAC2010} to gain an understanding of the optimization of strongly non-convex loss functions of neural networks. These results show that for high-dimensional (large $\Lambda$) spherical spin-glass models the lowest critical values of the Hamiltonians of these models form a layered structure and are located in a well-defined band lower-bounded by the global minimum. Simultaneously, the probability of finding them outside the band diminishes exponentially with the dimension of the spin-glass model. We next present the details of these results in the context of neural networks. We first introduce the notation and definitions. 
\begin{definition}
Let $u \in \mathbb{R}$ and $k$ be an integer such that $0 \leq k < \Lambda$. We will denote as $\mathcal{C}_{\Lambda,k}(u)$ a random number of critical values of $\mathcal{L}_{\Lambda,H}({\bm w})$ in the set $\Lambda B = \{\Lambda X:x\in (-\infty,u)\}$ with index\footnote{The number of negative eigenvalues of the Hessian $\nabla^2\mathcal{L}_{\Lambda,H}$ at ${\bm w}$ is also called index of $\nabla^2\mathcal{L}_{\Lambda,H}$ at ${\bm w}$.} equal to $k$. Similarly we will denote as $\mathcal{C}_{\Lambda}(B)$ a random total number of critical values of $\mathcal{L}_{\Lambda,H}(w)$.
\end{definition}
Later in the paper by critical values of the loss function that have non-diverging (fixed) index, or low-index, we mean the ones with index non-diverging with $\Lambda$.

\paragraph{The existence of the band of low-index critical points}
One can directly use Theorem 2.12 in~\cite{AAC2010} to show that for large-size networks (more precisely when $\Lambda \rightarrow \infty$ but recall that $\Lambda \rightarrow \infty$ iff $N \rightarrow \infty$) it is improbable to find a critical value below certain level $-\Lambda E_0(H)$ (which we call the \textit{ground state}), where $E_0(H)$ is some real number. 

Let us also  introduce the number that we will refer to as $E_{\infty}$. We will refer to this important threshold as the \textit{energy barrier} and define it as
\[E_{\infty} = E_{\infty}(H) = 2\sqrt{\frac{H-1}{H}}.
\]
Theorem 2.14 in~\cite{AAC2010} implies that for large-size networks all critical values of the loss function that are of non-diverging index must lie below the threshold $-\Lambda E_{\infty}(H)$. Any critical point that lies above the energy barrier is a high-index saddle point with overwhelming probability. Thus for large-size networks all critical values of the loss function that are of non-diverging index must lie in the band $\left(-\Lambda E_0(H),-\Lambda E_{\infty}(H)\right)$.

\begin{figure*}[htp!]
  \center
\includegraphics[width = 1.65in]{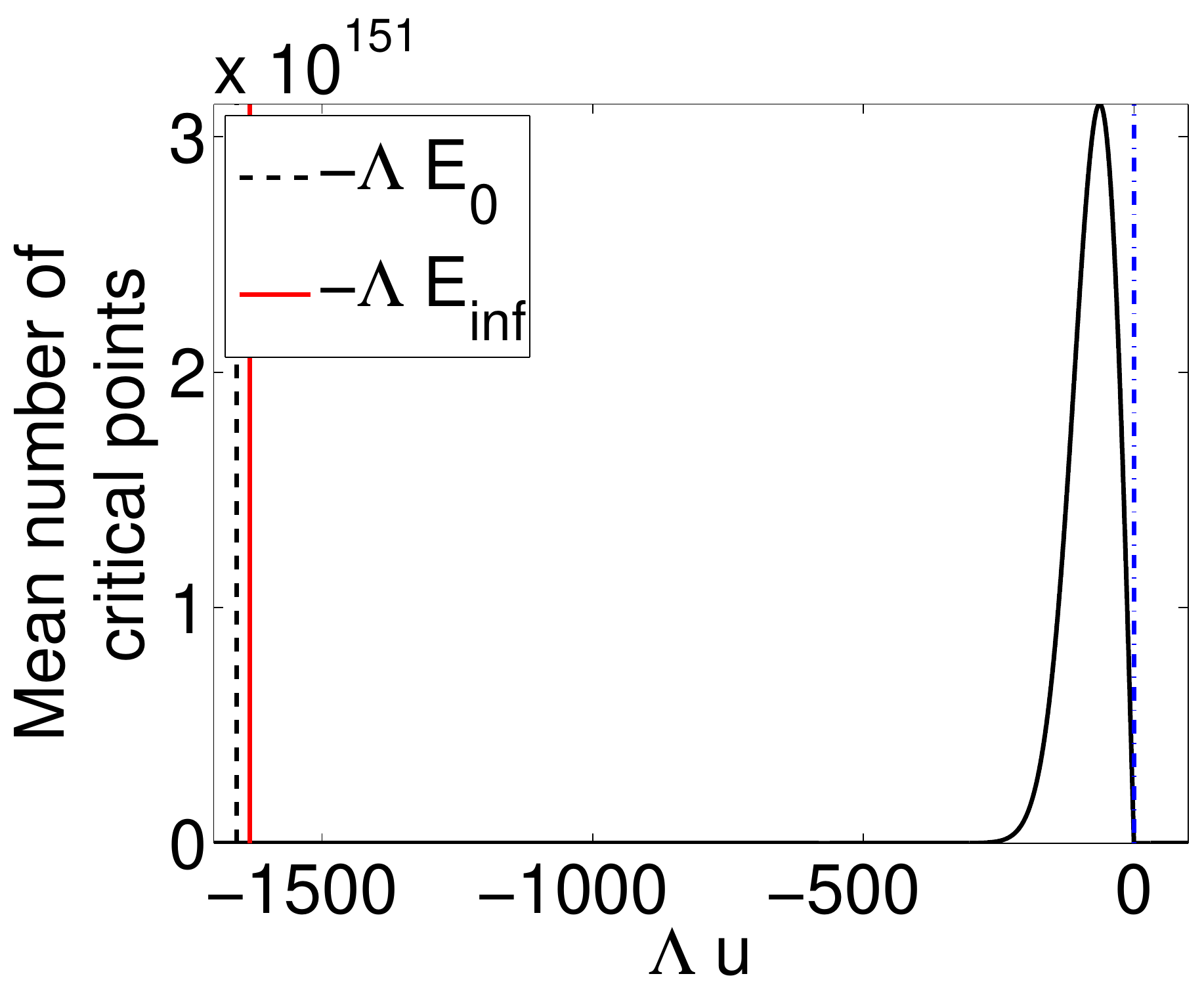}
\includegraphics[width = 1.65in]{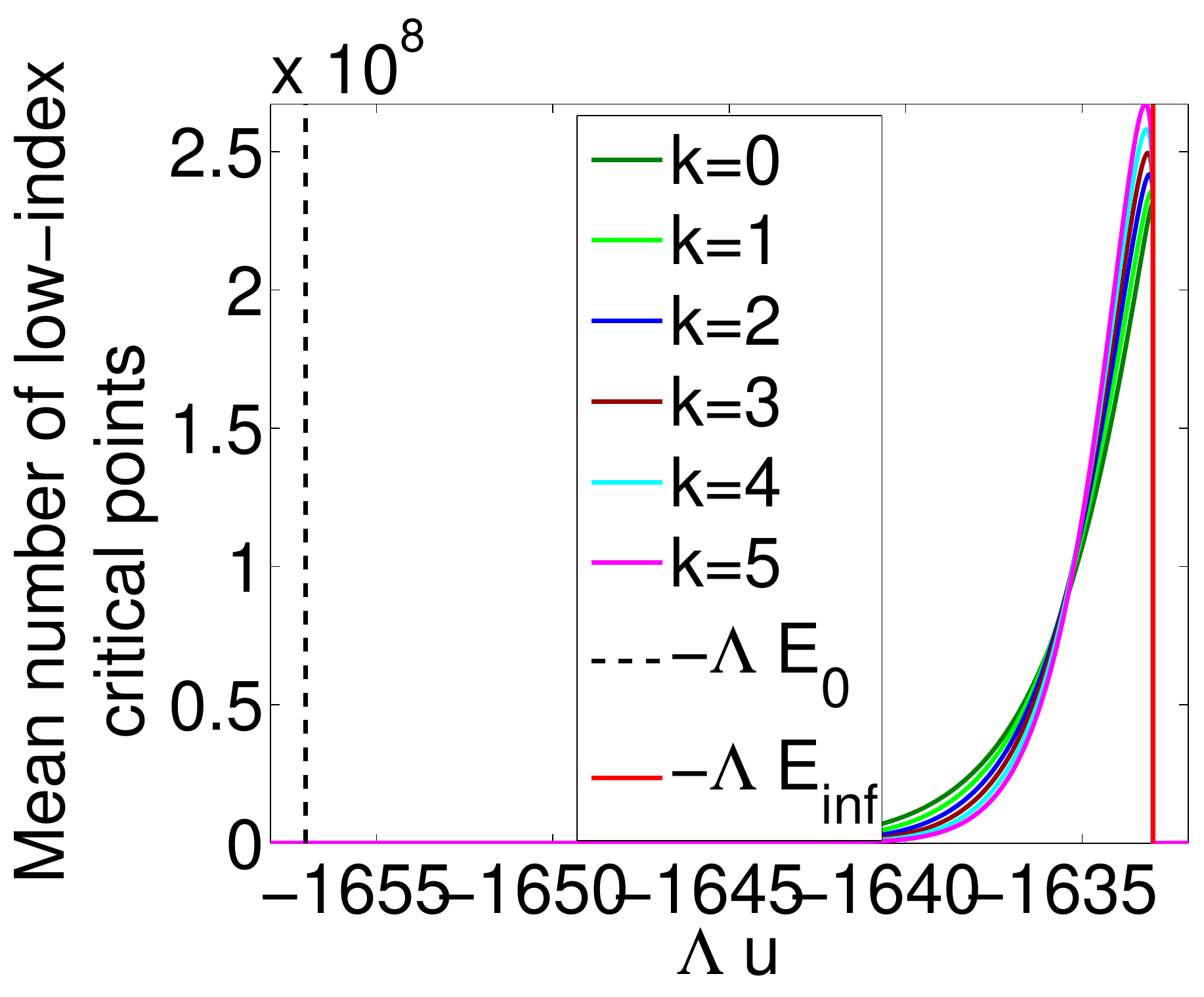} 
\includegraphics[width = 1.65in]{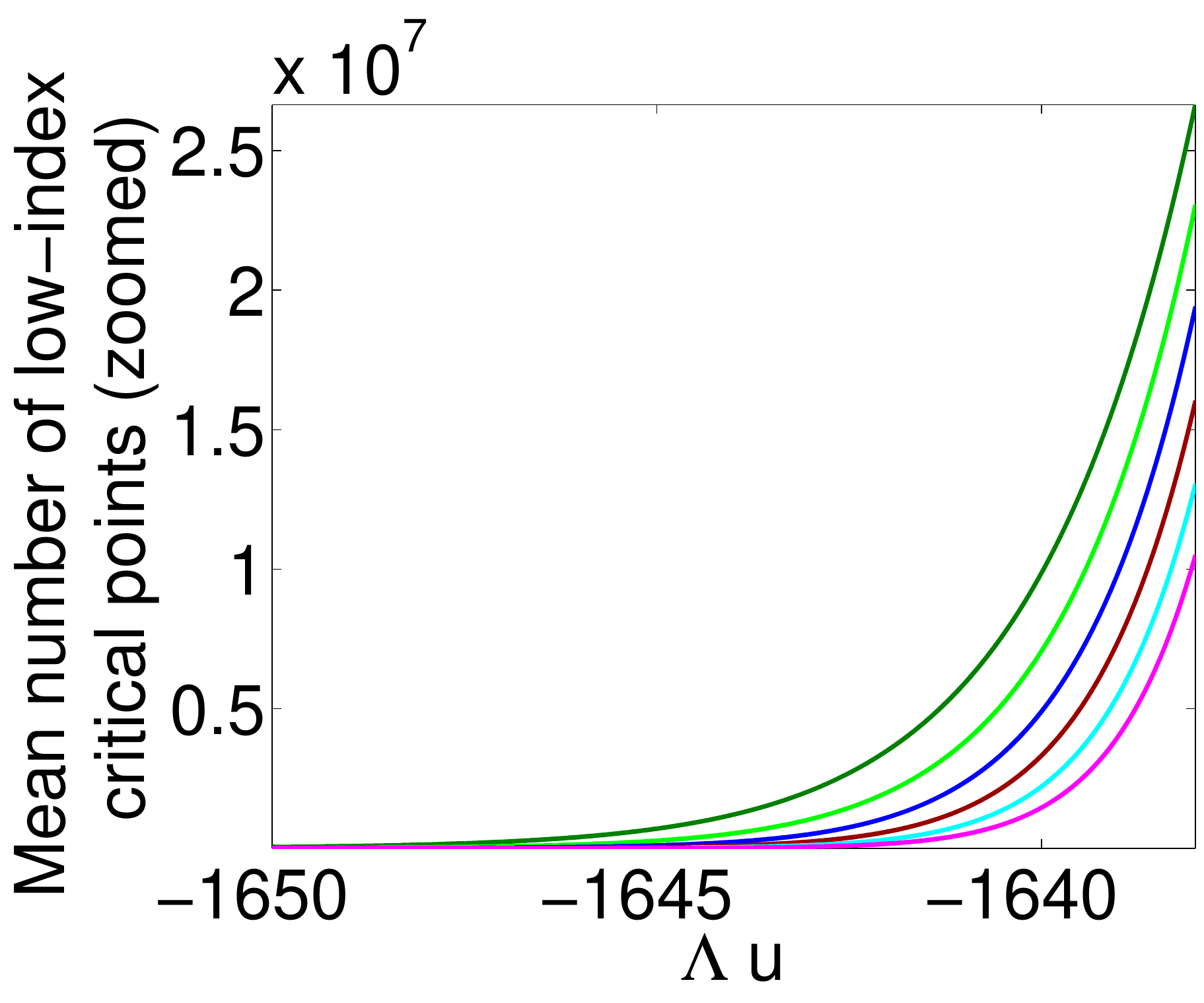} 
\includegraphics[width = 1.65in]{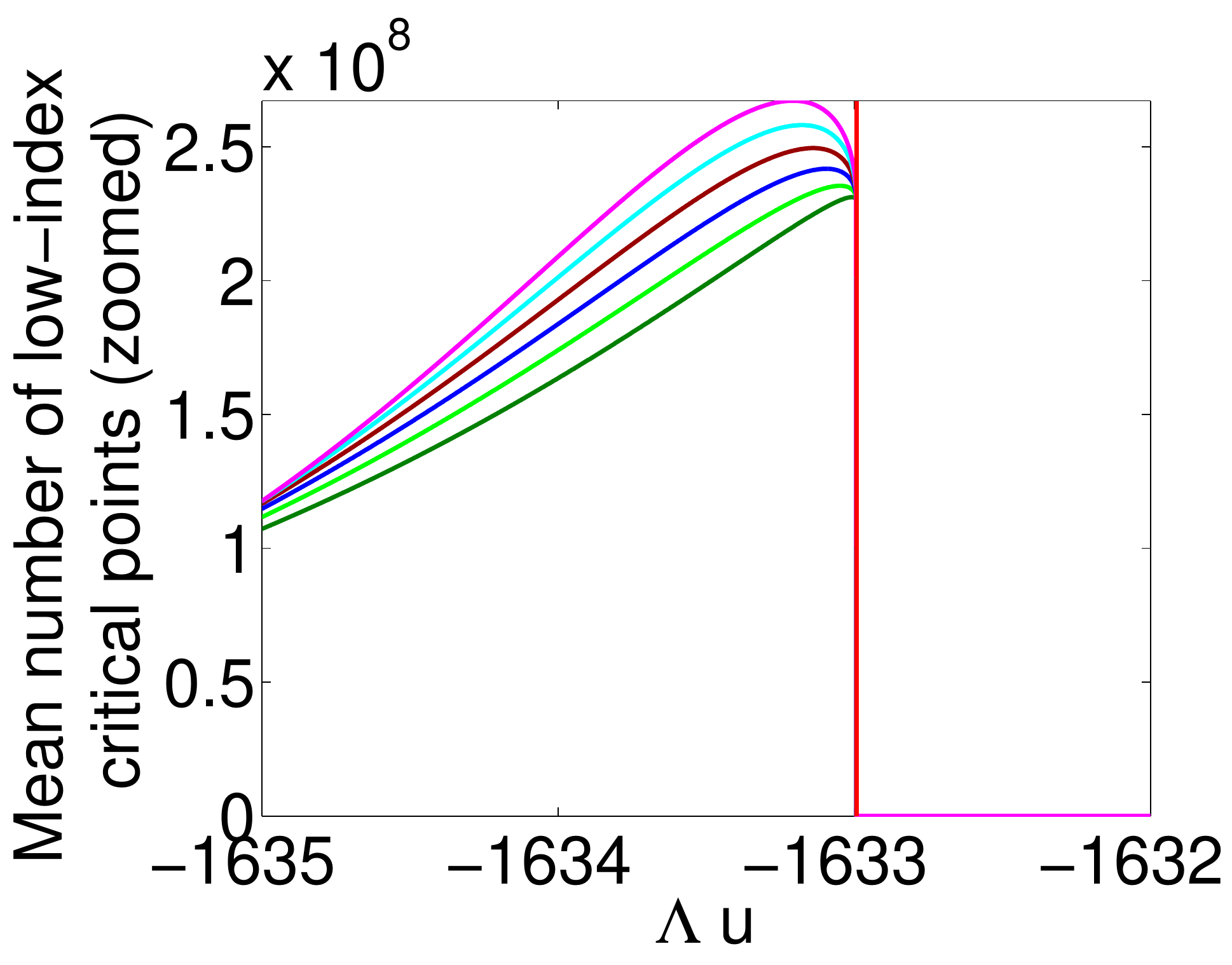} 
\vspace{-0.3in}
\caption{Distribution of the mean number of critical points, local minima and low-index saddle points (original and zoomed). Parameters $H$ and $\Lambda$ were set to $H = 3$ and $\Lambda = 1000$. Black line: $u = -\Lambda E_0(H)$, red line: $u = -\Lambda E_{\infty}(H)$. Figure must be read in color.}
\label{fig:Distr_cp_lm_sp}
\vspace{-0.15in}
\end{figure*}

\paragraph{Layered structure of low-index critical points}
From Theorem 2.15 in~\cite{AAC2010} it follows that for large-size networks finding a critical value with index larger or equal to $k$ (for any fixed integer $k$) below energy level $-\Lambda E_k(H)$ is improbable, where $-E_k(H) \in [-E_0(H),-E_{\infty}(H)]$. Furthermore, the sequence $\{E_k(H)\}_{k\in\mathbb{N}}$ is strictly decreasing and converges to $E_{\infty}$ as $k \rightarrow \infty$~\cite{AAC2010}. 

These results unravel a layered structure for the lowest critical values of the loss function of a large-size network, where with overwhelming probability the critical values above the global minimum (ground state) of the loss function are local minima exclusively. Above the band ($\left(-\Lambda E_0(H),-\Lambda E_1(H)\right)$) containing only local minima (critical points of index $0$), there is another one, ($\left(-\Lambda E_1(H),-\Lambda E_2(H)\right)$), where one can only find local minima and saddle points of index $1$, and above this band there exists another one, ($\left(-\Lambda E_2(H),-\Lambda E_3(H)\right)$), where one can only find local minima and saddle points of index $1$ and $2$, and so on. 

\paragraph{Logarithmic asymptotics of the mean number of critical points}
We will now define two non-decreasing, continuous functions on $\mathbb{R}$, $\Theta_H$ and $\Theta_{k,H}$ (their exemplary plots are captured in Figure~\ref{fig:Thetas}).
\[\Theta_H(u) = \left \{
  \begin{tabular}{c}
  $\!\!\!\!\frac{1}{2}\log(H\!-\!1) \!-\! \frac{(H-2)u^2}{4(H-1)} \!-\! I(u)$ $\:\:\:$if$\:\:$$u \!\leq\! -E_{\infty}$\\
  $\!\!\!\!\!\frac{1}{2}\log(H\!-\!1) \!-\! \frac{(H-2)u^2}{4(H-1)}$ $\:\:\:\:\:\:$if$\:\:$$-E_{\infty} \!\leq\! u \leq 0$\\
  $\!\!\!\!\!\!\!\!\!\!\!\!\!\!\!\frac{1}{2}\log(H-1)$ $\:\:\:\:\:\:\:\:\:\:\:\:\:\:\:\:\:\:\:\:\:\:\:\:\:\:\:\:\:\:\:$if$\:\:$$0 \!\leq\! u$
  \end{tabular}
\right.,
\]
and for any integer $k \geq 0$:
\[\Theta_{k,H}(u) \!=\! \left \{
  \begin{tabular}{c}
  $\!\!\!\!\!\frac{1}{2}\!\log(\!H\!-\!1\!) \!-\! \frac{(H\!-\!2)u^2}{4(\!H-1\!)} \!-\! (k\!+\!1)I(u)$ if$\:$$u \!\!\leq\!\! -E_{\infty}$\\
  $\!\!\!\!\frac{1}{2}\log(\!H\!-\!1\!) \!-\! \frac{H\!-\!2}{4(\!H-1\!)}$ $\:\:\:\:\:\:\:\:\:\:\:\:\:\:\:\:\:\:\:\:\:\:\:\:$if$\:$$u \!\!\geq\!\! -E_{\infty}$
  \end{tabular}
\right.
\]
where 
\[I(u) = -\frac{u}{E_{\infty}^2}\sqrt{u^2 \!-\! E_{\infty}^2} - \log(-u + \sqrt{u^2 \!-\! E_{\infty}^2}) + \log E_{\infty}.
\] 

\begin{figure}[h]
  \center
\includegraphics[width = 1.6in]{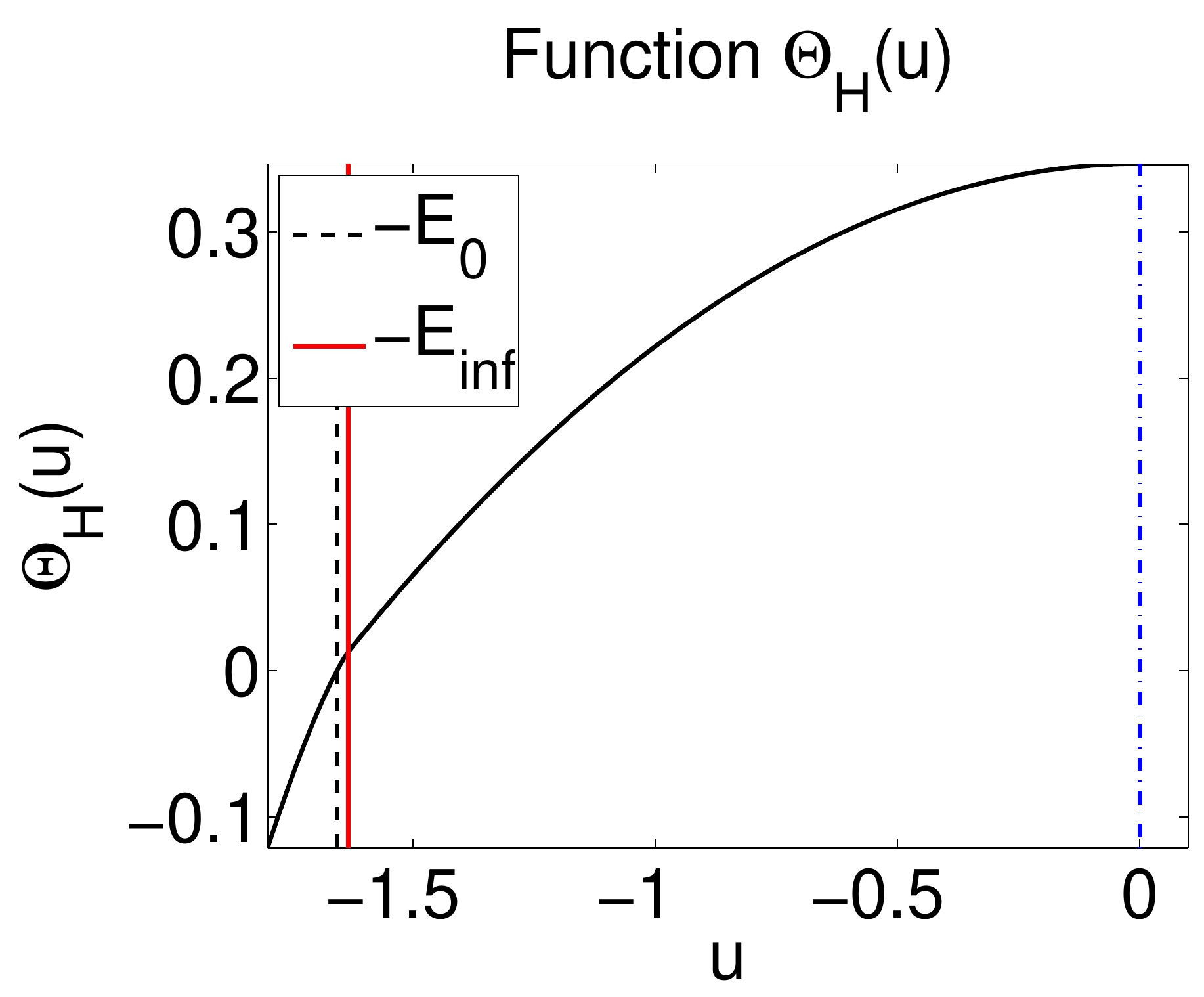}
\includegraphics[width = 1.6in]{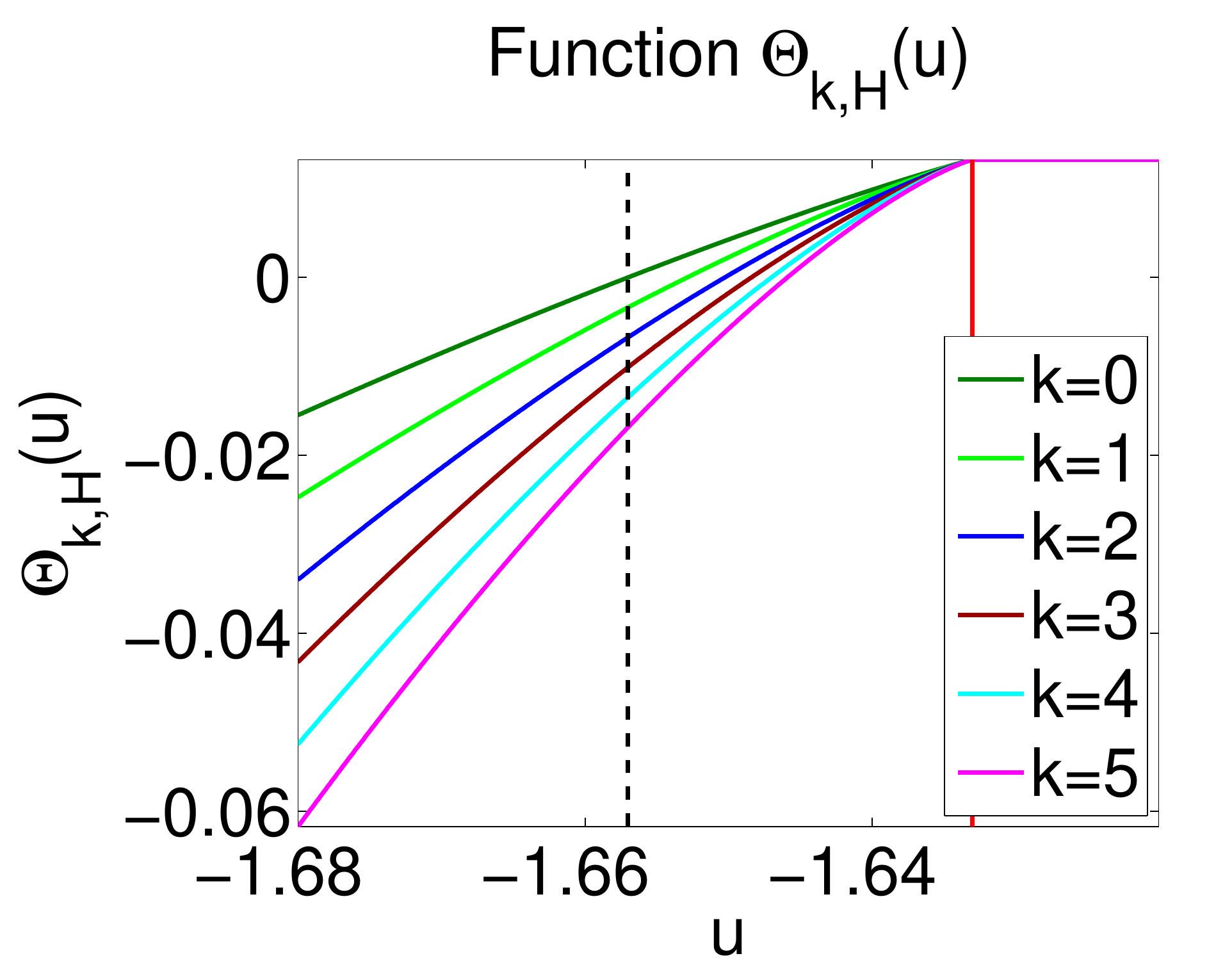}
\vspace{-0.3in}
\caption{Functions $\Theta_H(u)$ and $\Theta_{k,H}(u)$ for $k = \{0,1,\dots,5\}$. Parameter $H$ was set to $H = 3$. Black line: $u = -E_0(H)$, red line: $u = -E_{\infty}(H)$. Figure must be read in color.}
\label{fig:Thetas}
\vspace{-0.1in}
\end{figure}

Also note that the following corollary holds.

\begin{corollary}
For all $k > 0$ and $u < -E_{\infty}$, $\Theta_{k,H}(u) < \Theta_{0,H}(u)$.
\label{cor:Theta}
\end{corollary}

Next we will show the logarithmic asymptotics of the mean number of critical points (the asymptotics of the mean number of critical points can be found in the Supplementary material). 
\begin{theorem}[\cite{AAC2010}, Theorem 2.5 and 2.8]
For all $H \geq 2$
\[\lim_{\Lambda \rightarrow \infty}\frac{1}{\Lambda}\log\mathbb{E}[\mathcal{C}_{\Lambda}(u)] = \Theta_{H}(u).
\]
and for all $H \geq 2$ and $k \geq 0$ fixed
\[\lim_{\Lambda \rightarrow \infty}\frac{1}{\Lambda}\log\mathbb{E}[\mathcal{C}_{\Lambda,k}(u)] = \Theta_{k,H}(u).
\]
\label{thm:lmsp}
\end{theorem}
From Theorem~\ref{thm:lmsp} and Corollary~\ref{cor:Theta} the number of critical points in the band $\left(-\Lambda E_0(H),-\Lambda E_{\infty}(H)\right)$ increases exponentially as $\Lambda$ grows and that local minima dominate over saddle points and this domination also grows exponentially as $\Lambda$ grows. Thus for large-size networks the probability of recovering a saddle point in the band $\left(-\Lambda E_0(H),-\Lambda E_{\infty}(H)\right)$, rather than a local minima, goes to $0$.

Figure~\ref{fig:Distr_cp_lm_sp} captures exemplary plots of the distributions of the mean number of critical points, local minima and low-index saddle points. Clearly local minima and low-index saddle points are located in the band $\left(-\Lambda E_0(H),-\Lambda E_{\infty}(H)\right)$ whereas high-index saddle points can only be found above the energy barrier $-\Lambda E_{\infty}(H)$. Figure~\ref{fig:Distr_cp_lm_sp} also reveals the layered structure for the lowest critical values of the loss function\footnote{The large mass of saddle points above $-\Lambda E_{\infty}$ is a consequence of Theorem~\ref{thm:lmsp} and the properties of $\Theta$ functions.}. This 'geometric' structure plays a crucial role in the optimization problem. The optimizer, e.g. SGD, easily avoids the band of high-index critical points, which have many negative curvature directions, and descends to the band of low-index critical points which lie closer to the global minimum. Thus finding bad-quality solution, i.e. the one far away from the global minimum, is highly unlikely for large-size networks.

\paragraph{Hardness of recovering the global minimum}
Note that the energy barrier to cross when starting from any (local) minimum, e.g. the one from the band $\left(-\Lambda E_i(H),-\Lambda E_{i+1}(H)\right)$, in order to reach the global minimum diverges with $\Lambda$ since it is bounded below by $\Lambda (E_0(H) - E_i(H))$. Furthermore, suppose we are at a local minima with a scaled energy of $-E_{\infty} - \delta$. 
In order to find a further low lying minimum we must pass through a saddle point. 
Therefore we must go up at least to the level where there is an equal amount of saddle points to have a decent chance of finding a path that might possibly take us to another local minimum. 
This process takes an exponentially long time so in practice finding the global minimum is not feasible.

Note that the variance of the loss in Equation \ref{eq:spinglass} is $\Lambda$ which suggests that the extensive quantities should scale with $\Lambda$. 
In fact this is the reason behind the scaling factor in front of the summation in the loss.
The relation to the logarithmic asymptotics is as follows: the number of critical values of the loss below the level $\Lambda u$ is roughly $e^{\Lambda\Theta_H (u)}$. 
The gradient descent gets trapped roughly at the barrier denoted by $-\Lambda E_{\infty}$, as will be shown in the experimental section. 

\begin{figure*}%[htp!]
  \center
\includegraphics[width = 3.0in]{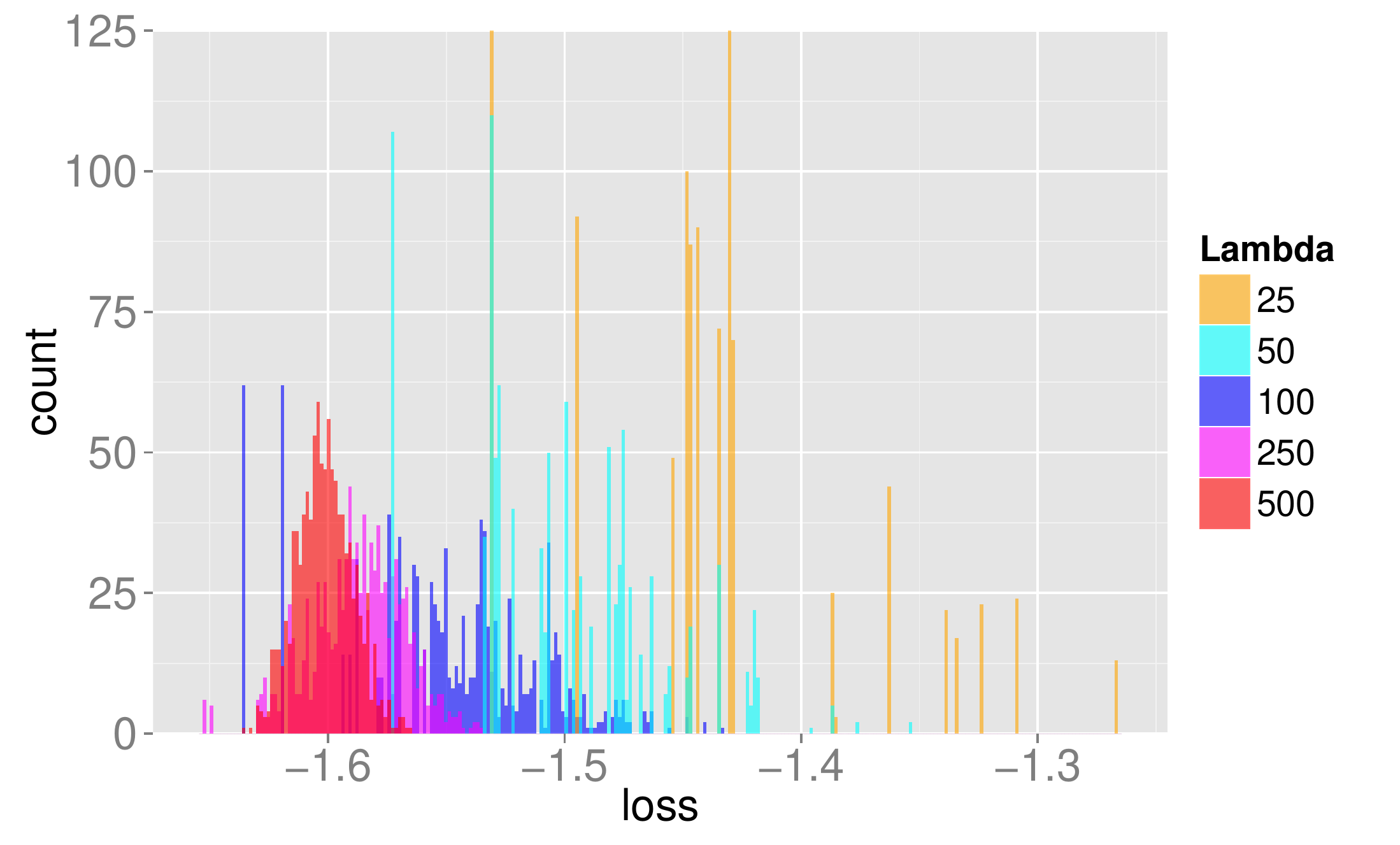}
\includegraphics[width = 3.0in]{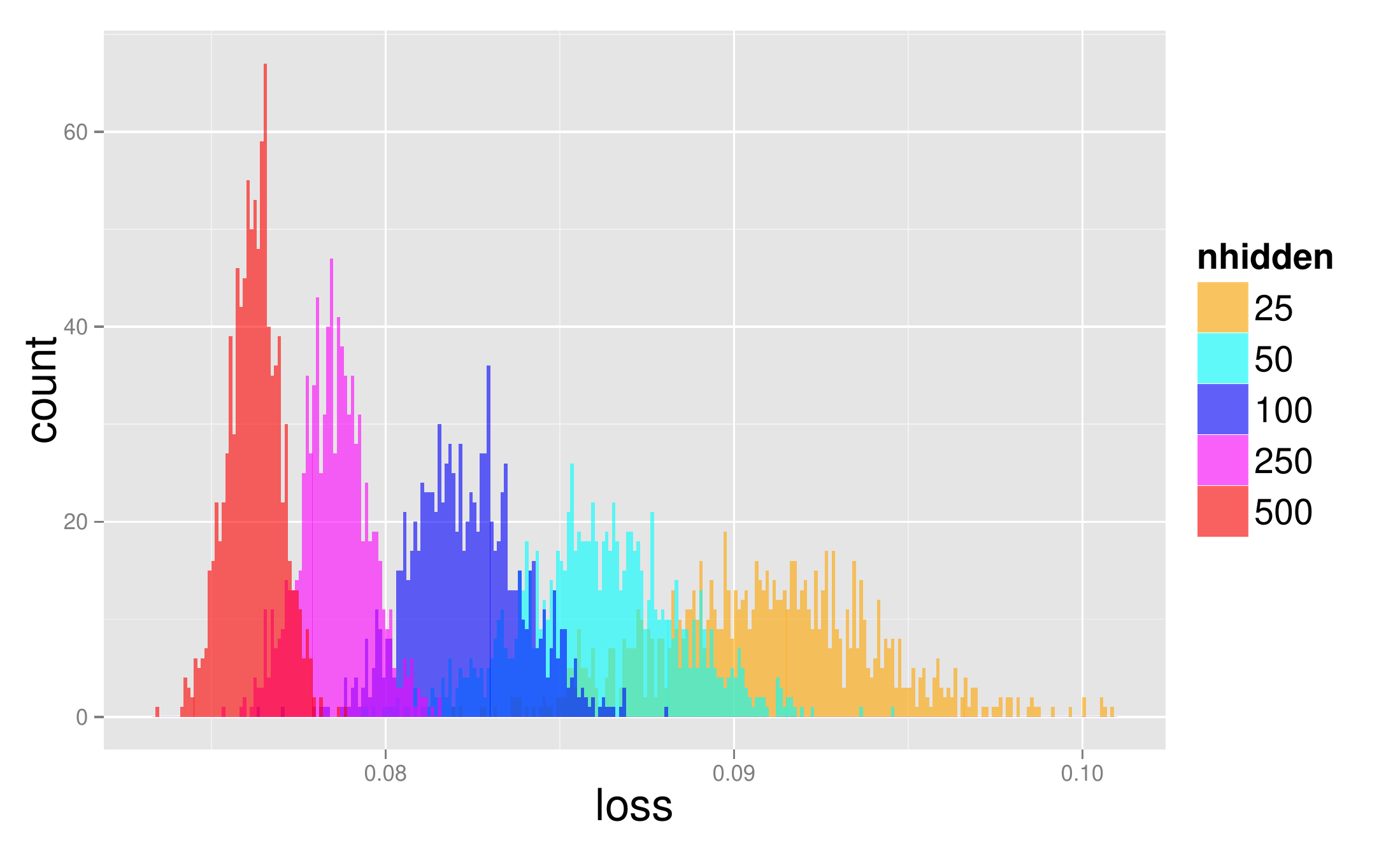} 
\vspace{-0.15in}
\caption{Distributions of the scaled test losses for the spin-glass (left) and the neural network (right) experiments.}
\label{fig:spinglass_and_mnist}
\vspace{-0.17in}
\end{figure*}

\section{Experiments}
\label{sec:Experiments}

The theoretical part of the paper considers the problem of training the neural network, whereas the empirical results focus on its generalization properties.

\subsection{Experimental Setup}

\paragraph{Spin-Glass}

To illustrate the theorems in Section 4, we conducted spin-glass simulations for different dimensions $\Lambda$ from 25 to 500. 
For each value of $\Lambda$, we obtained an estimate of the distribution of minima by sampling 1000 initial points on the unit sphere and performing stochastic gradient descent (SGD) to find a minimum energy point.
Note that throughout this section we will refer to the energy of the Hamiltonian of the spin-glass model as its loss.

\paragraph{Neural Network}

We performed an analogous experiment on a scaled-down version of MNIST, where each image was downsampled to size $10 \times 10$.
%To validate this, we trained a large number of one-layer networks with different numbers of hidden units on a scaled-down version of MNIST, where each image was downsampled to be $10 \times 10$.
Specifically, we trained 1000 networks with one hidden layer and $n_1 \in \{25,50,100,250,500 \}$ hidden units (in the paper we also refer to the number of hidden units as \textit{nhidden}), each one starting from a random set of parameters sampled uniformly within the unit cube. 
All networks were trained for 200 epochs using SGD with learning rate decay.  

To verify the validity of our theoretical assumption of parameter redundancy, we also trained a neural network on a subset of MNIST using simulated annealing (SA) where $95\%$ of parameters were assumed to be redundant.
Specifically, we allowed the weights to take one of $3$ values uniformly spaced in the interval $[-1,1]$. 
We obtained less than $2.5\%$ drop in accuracy, which demonstrates the heavy over-parametrization of neural networks as discussed in Section \ref{sec:NNSG}.

\paragraph{Index of critical points}
It is necessary to verify that our solutions obtained through SGD are low-index critical points rather than high-index saddle points of poor quality. As observed by \cite{DBLP:journals/corr/DauphinPGCGB14} certain optimization schemes have the potential to get trapped in the latters. We ran two tests to ensure that this was not the case in our experimental setup. 
First, for $n_1=\{10,25,50,100\}$ we computed the eigenvalues of the Hessian of the loss function at each solution and computed the index. 
All eigenvalues less than 0.001 in magnitude were set to 0. Figure~\ref{fig:index_distr} captures an exemplary distribution of normalized indices, which is the proportion of negative eigenvalues, for $n_1 = 25$ (the results for $n_1=\{10,50,100\}$ can be found in the Supplementary material). It can be seen that all solutions are either minima or saddle points of very low normalized index (of the order 0.01). Next, we compared simulated annealing to SGD on a subset of MNIST. Simulated annealing does not compute gradients and thus does not tend to become trapped in high-index saddle points. We found that SGD performed at least as well as simulated annealing, which indicates that becoming trapped in poor saddle points is not a problem in our experiments. The result of this comparison is in the Supplementary material. All figures in this paper should be read in color.
\begin{figure}[h]
\vspace{-0.1in}
\includegraphics[trim=0cm 0cm 0cm 2.4cm,clip,scale=0.4]{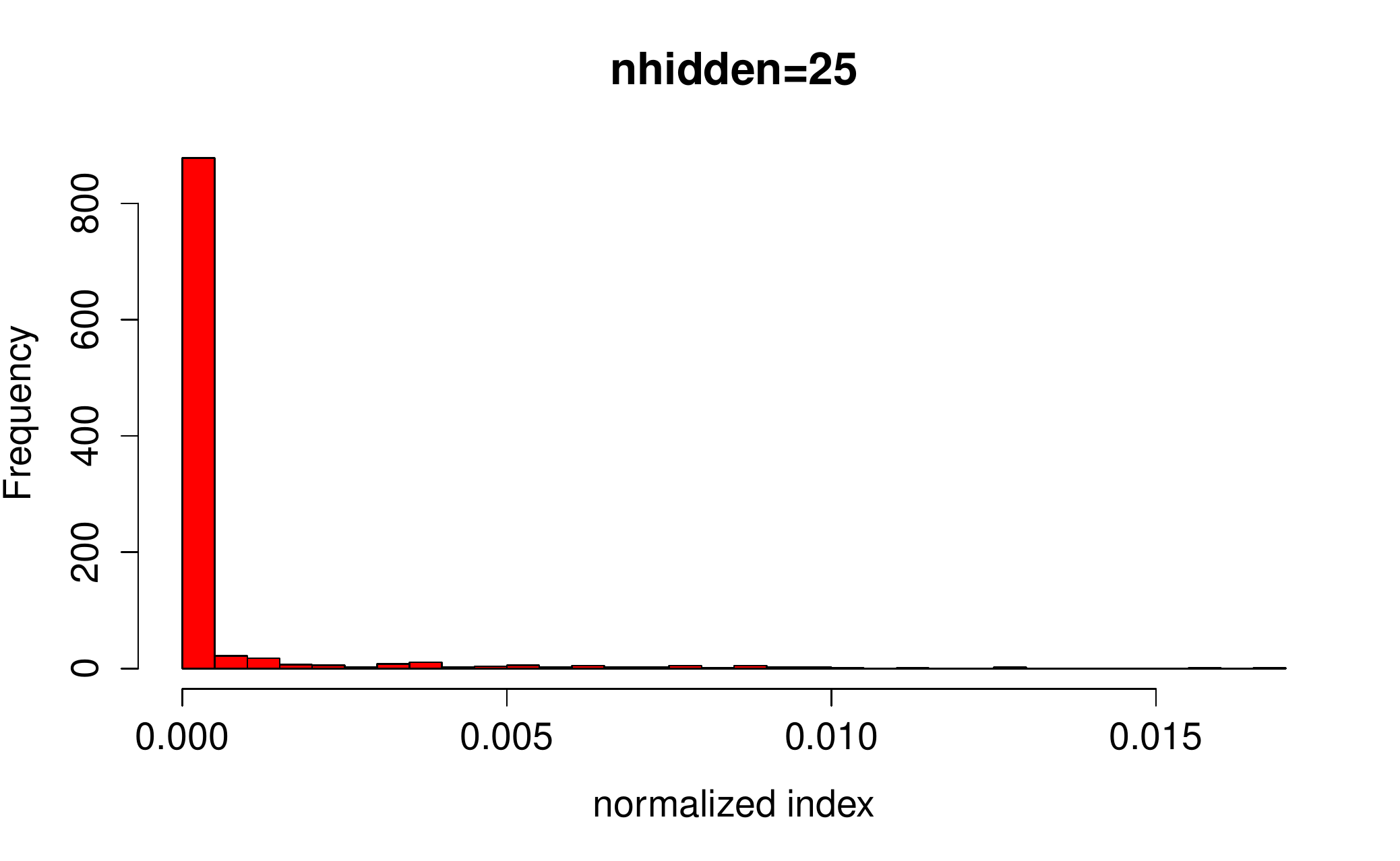}
\vspace{-0.35in}
\caption{Distribution of normalized index of solutions for 25 hidden units.}
\label{fig:index_distr}
\vspace{-0.25in}
\end{figure} 

\paragraph{Scaling loss values}
To observe qualitative differences in behavior for different values of $\Lambda$ or $n_1$, it is necessary to rescale the loss values to make their expected values approximately equal.
For spin-glasses, the expected value of the loss at critical points scales linearly with $\Lambda$, therefore we divided the losses by $\Lambda$ (note that this normalization is in the statement of Theorem \ref{thm:lmsp}) which gave us the histogram of points at the correct scale. 
For MNIST experiments, we empirically found that the loss with respect to number of hidden units approximately follows an exponential power law: $ \mathbb{E}[L] \propto e^{\alpha n_1^{\beta}}$.
We fitted the coefficients $\alpha,\beta$ and scaled the loss values to $L/e^{\alpha n_1^{\beta}}$.

%To observe qualitative differences in behavior for different size networks, it is necessary to scale the loss values to make their expected values approximately equal.

\subsection{Results}
Figure 3 shows the distributions of the scaled test losses for both sets of experiments. 
For the spin-glasses (left plot), we see that for small values of $\Lambda$, we obtain poor local minima on many experiments, while for larger values of $\Lambda$ the distribution becomes increasingly concentrated around the energy barrier where local minima have high quality. 
We observe that the left tails for all $\Lambda$ touches the barrier that is hard to penetrate and as $\Lambda$ increases the values concentrate around $-E_{\infty}$. In fact this concentration result has long been predicted but not proved until \cite{AAC2010}. We see that qualitatively the distribution of losses for the neural network experiments (right plot) exhibits similar behavior.
Even after scaling, the variance decreases with higher network sizes. This is also clearly captured in Figure~\ref{fig:spinglass_and_mnist2} and~\ref{fig:mnist_statistics1} in the Supplementary material. This indicates that getting stuck in poor local minima is a major problem for smaller networks but becomes gradually of less importance as the network size increases.
This is because critical points of large networks exhibit the layered structure where high-quality low-index critical points lie close to the global minimum.

\subsection{Relationship between train and test loss}

\begin{table}[h]
\centering
\begin{tabular}{|c|c|c|c|c|c|}
\hline
 $n_1$ & 25 & 50 & 100 & 250 & 500 \\
 \hline
 $\rho$ & 0.7616 & 0.6861 & 0.5983 & 0.5302 & 0.4081 \\
\hline
\end{tabular}
%\vspace{-0.05in}
\caption{Pearson correlation between training and test loss for different numbers of hidden units.}
\label{tab:correlations}
\vspace{-0.05in}
\end{table}
The theory and experiments thus far indicate that minima lie in a band which gets smaller as the network size increases. 
This indicates that computable solutions become increasingly equivalent with respect to training error, but how does this relate to error on the test set? 
To determine this, we computed the correlation $\rho$ between training and test loss for all solutions for each network size. The results are captured in Table~\ref{tab:correlations} and Figure~\ref{fig:correlations} (the latter is in the Supplementary material). The training and test error become increasingly decorrelated as the network size increases. This provides further indication that attempting to find the absolute possible minimum is of limited use with regards to generalization performance. 

\section{Conclusion}
\label{sec:ConandFutWork}

This paper establishes a connection between the neural network and the spin-glass model. We show that under certain assumptions, the loss function of the fully decoupled large-size neural network of depth $H$ has similar landscape to the Hamiltonian of the $H$-spin spherical spin-glass model. 
We empirically demonstrate that both models studied here are highly similar in real settings, despite the presence of variable dependencies in real networks. 
To the best of our knowledge our work is one of the first efforts in the literature to shed light on the theory of neural network optimization.

\subsubsection*{Acknowledgements}
The authors thank L. Sagun and the referees for valuable feedback.

\bibliographystyle{apalike}
\bibliography{PAPER_AMMGY}

\clearpage

\toptitlebar 
{\Large \bf  \centering{The Loss Surfaces of Multilayer Networks\\(Supplementary Material)} \par}
\bottomtitlebar

\section{Proof of Theorem~\ref{thm:arge}}
\begin{proof}
First we will prove the lower-bound on $N$. By the inequality between arithmetic and geometric mean the mass and the size of the network are connected as follows
\[N \geq \sqrt[H]{\Psi^2}\frac{H}{\sqrt[H]{n_0n_H}} = \sqrt[H]{\Psi^2}\frac{H}{\sqrt[H]{d}},
\]
and since $\sqrt[H]{\Psi}\frac{H}{\sqrt[H]{d}} = \sqrt[H]{\prod_{i = 1}^{H}n_i}H \geq 1$ then 
\[N \geq \sqrt[H]{\Psi^2}\frac{H}{\sqrt[H]{d}} \geq \sqrt[H]{\Psi}.
\]
Next we show the upper-bound on $N$. Let $n_{max} = \max_{i \in \{1,2,\dots,H\}}n_i$. Then
\[N \leq Hn_{max}^2 \leq H\Psi^2.
\]
\end{proof}

\section{Proof of Theorem~\ref{thm:redun}}
\begin{proof}
We will first proof the following more general lemma.
\begin{lemma}
Let $Y_1$ and $Y_2$ be the outputs of two arbitrary binary classifiers. Assume that the first classifiers predicts $1$ with probability $p$ where, without loss of generality, we assume $p \leq 0.5$ and $-1$ otherwise. Furthemore, let the prediction accuracy of the second classifier differ from the prediction accuracy of the first classifier by no more than $\epsilon \in [0,p]$. Then the following holds
\[corr(sign(Y_1),sign(Y_2))
\]
\[\geq \frac{1 - 2\epsilon - (1 - 2p)^2 - 2(1 - 2p)\epsilon}{4\sqrt{p(1 - p)(p + \epsilon)(1 - p + \epsilon)}}.
\]
\end{lemma}
\begin{proof}
Consider two random variables $Z_1 = sign(Y_1)$ and $Z_2 = sign(Y_2)$. Let $\mathcal{X}^{+}$ denote the set of data points for which the first classifier predicts $+1$ and let $\mathcal{X}^{-}$ denote the set of data points for which the first classifier predicts $-1$ ($\mathcal{X}^{+} \cup \mathcal{X}^{-} = \mathcal{X}$, where $\mathcal{X}$ is the entire dataset). Also let $p = \frac{|\mathcal{X}^{+}|}{|\mathcal{X}|}$. Furthermore, let $\mathcal{X}_{\epsilon}^{-}$ denote the dataset for which $Z_1 = +1$ and $Z_2 = -1$ and $\mathcal{X}_{\epsilon}^{+}$ denote the dataset for which $Z_1 = -1$ and $Z_2 = +1$, where $\frac{|\mathcal{X}_{\epsilon}^{+}| + |\mathcal{X}_{\epsilon}^{-}|}{\mathcal{X}} = \epsilon$. Also let $\epsilon^{+} = \frac{|\mathcal{X}_{\epsilon}^{+}|}{\mathcal{X}}$ and $\epsilon^{-} = \frac{|\mathcal{X}_{\epsilon}^{-}|}{\mathcal{X}}$. Therefore
\[Z_1 \!=\! \left \{
  \begin{tabular}{c}
  $1$ \:\:if $x \in \mathcal{X}^{+}$\\
  \!\!\!\!$-1$ \:\:if $x \in \mathcal{X}^{-}$
  \end{tabular}
\right.
\]
and
\[Z_2 \!=\! \left \{
  \begin{tabular}{c}
  $1$ \:\:if $x \in \mathcal{X}^{+} \cup \mathcal{X}_{\epsilon}^{+}  \setminus \mathcal{X}_{\epsilon}^{-}$\\
  \!\!\!$-1$ \:\:if $x \in \mathcal{X}^{-} \cup \mathcal{X}_{\epsilon}^{-} \setminus \mathcal{X}_{\epsilon}^{+}$.
  \end{tabular}
\right.
\]
One can compute that $\mathbb{E}[Z_1] = 2p-1$, $\mathbb{E}[Z_2] = 2(p + \epsilon^{+} - \epsilon^{-}) - 1$, $\mathbb{E}[Z_1Z_2] = 1 - 2\epsilon$, $std(Z_s) = 2\sqrt{p(1-p)}$, and finally $std(Z_\Lambda) = 2\sqrt{(p + \epsilon^{+} - \epsilon^{-})(1 - p - \epsilon^{+} + \epsilon^{-})}$.
Thus we obtain
\begin{eqnarray}
&&\!\!\!\!\!\!\!\!\!\!\!\!\!\!\!corr(sign(Y_1),sign(Y_2)) = corr(Z_1,Z_2) \nonumber\\
&\!\!\!\!\!\!\!\!\!\!\!\!\!\!\!=& \!\!\!\!\!\!\!\!\!\!\frac{\mathbb{E}[Z_1Z_2] - \mathbb{E}[Z_1]\mathbb{E}[Z_2]}{std(Z_1)std(Z_2)} \nonumber\\
&\!\!\!\!\!\!\!\!\!\!\!\!\!\!\!=& \!\!\!\!\!\!\!\!\!\!\frac{1 - 2\epsilon - (1 - 2p)^2 + 2(1 - 2p)(\epsilon^{+} - \epsilon^{-})}{4\sqrt{p(1 - p)(p + \epsilon^{+} - \epsilon^{-})(1 - p - \epsilon^{+} + \epsilon^{-})}} \nonumber\\
&\!\!\!\!\!\!\!\!\!\!\!\!\!\!\!\geq& \!\!\!\!\!\!\!\!\!\!\frac{1 - 2\epsilon - (1 - 2p)^2 - 2(1 - 2p)\epsilon}{4\sqrt{p(1 - p)(p + \epsilon)(1 - p + \epsilon)}}
\label{eq:tmp}
\end{eqnarray}
\end{proof}
Note that when the first classifier is network $\mathcal{N}$ considered in this paper and $\mathcal{M}$ is its $(s,\epsilon)$-reduction image $\mathbb{E}[Y_1] = 0$ and $\mathbb{E}[Y_2] = 0$ (that follows from the fact that $X$'s in Equation~\ref{eq:befapprox} have zero-mean). That implies $p=0.5$ which, when substituted to Equation~\ref{eq:tmp} gives the theorem statement.
\end{proof}

\section{Proof of Theorem~\ref{thm:unif}}
\begin{proof}
Note that $\mathbb{E}[\hat{Y}_s] = 0$ and $\mathbb{E}[Y_s] = 0$. Furthermore
\[\mathbb{E}[\hat{Y}_sY_s] = q^2\rho^2\!\!\!\sum_{i_1,i_2,\dots,i_H=1}^{s}\!\!\!\min\left(\frac{\Psi}{s^H},t_{i_1,i_2,\dots,i_H}\right)\prod_{k = 1}^{H}w_{i_k}^2
\]
and
\[std(\hat{Y}_s) = q\rho\sqrt{\sum_{i_1,i_2,\dots,i_H=1}^{s}\frac{\Psi}{s^H}\prod_{k = 1}^{H}w_{i_k}^2}
\]
\[std(Y_s) = q\rho\sqrt{\sum_{i_1,i_2,\dots,i_H=1}^{s}t_{i_1,i_2,\dots,i_H}\prod_{k = 1}^{H}w_{i_k}^2}.
\]
Therefore
\begin{eqnarray*}
&&\!\!\!\!\!\!\!\!\!corr(\hat{Y}_s,Y_s)\\ 
&\!\!\!\!\!\!\!\!\!=&\!\!\!\!\!\!\! \frac{\displaystyle\sum_{i_1,\dots,i_H=1}^{s}\min\left(\frac{\Psi}{s^H},t_{i_1,\dots,i_H}\right)\prod_{k = 1}^{H}w_{i_k}^2}{\sqrt{\!\!\!\left(\displaystyle\sum_{i_1,\dots,i_H=1}^{s}\frac{\Psi}{s^H}\prod_{k = 1}^{H}w_{i_k}^2\right)\!\!\!\!\left(\displaystyle\sum_{i_1,\dots,i_H=1}^{s}\!\!\!\!\!t_{i_1,\dots,i_H}\prod_{k = 1}^{H}w_{i_k}^2\right)}}\\
&\!\!\!\!\!\!\!\!\geq& \frac{1}{c^2}, 
\end{eqnarray*}
where the last inequality is the direct consequence of the uniformity assumption of Equation~\ref{eq:uniform}.
\end{proof}

\section{Loss function as a $H$ - spin spherical spin-glass model}

We consider two loss functions, (random) absolute loss $\mathcal{L}^a_{\Lambda,H}(w)$ and (random) hinge loss $\mathcal{L}^h_{\Lambda,H}(w)$ defined in the main body of the paper. Recall that in case of the hinge loss $\max$ operator can be modeled as Bernoulli random variable, that we will refer to as $M$, with success ($M = 1$) probability $\rho^{'} = \frac{\mathcal{C}^{'}}{\rho\sqrt{\mathcal{C}^H}}$ for some non-negative constant $\mathcal{C}^{'}$. We assume $M$ is independent of $\hat{Y}$. Therefore we obtain that
\[\mathcal{L}^a_{\Lambda,H}({\bm w}) = \left \{
  \begin{tabular}{c}
  $\!\!\!\!\!\!S - \hat{Y}$ $\:\:\:\:$if$\:\:\:\:$ $Y_t = S$ \\
  $\!S + \hat{Y}$ $\:\:\:\:$if$\:\:\:\:$ $Y_t = -S$ \\
  \end{tabular}
\right.
\]
and
\begin{eqnarray*}
\mathcal{L}^h_{\Lambda,H}({\bm w}) &=& \mathbb{E}_{M,A}[M(
1-Y_t\hat{Y})]\\&=& \left \{
  \begin{tabular}{c}
  $\!\!\!\!\!\!\mathbb{E}_{M}[M(1 - \hat{Y})]$ $\:\:\:\:$if$\:\:\:\:$ $Y_t = 1$ \\
  $\!\mathbb{E}_{M}[M(1 + \hat{Y})]$ $\:\:\:\:$if$\:\:\:\:$ $Y_t = -1$ \\
  \end{tabular}
\right.
\end{eqnarray*}
Note that both cases can be generalized as
\[\mathcal{L}_{\Lambda,H}({\bm w}) = \left \{
  \begin{tabular}{c}
  $\!\mathbb{E}_{M}[M(S - \hat{Y})]$ $\:\:\:\:$if$\:\:\:\:$ $Y_t  > 0$ \\
  $\!\mathbb{E}_{M}[M(S + \hat{Y})]$ $\:\:\:\:$if$\:\:\:\:$ $Y_t < 0$ \\
  \end{tabular}
\right.,
\]
where in case of the absolute loss $\rho^{'} = 1$ and in case of the hinge loss $S = 1$. Furthermore, using the fact that $X$'s are Gaussian random variables one we can further generalize both cases as
\[\mathcal{L}_{\Lambda,H}({\bm w}) = S\rho^{'} \!+\! q\sum_{i_1,i_2,\dots,i_H=1}^{\Lambda}\!\!\!\!\!\!\!X_{i_1,i_2,\dots,i_H}\rho\rho^{'}\prod_{k = 1}^{H}w_{i_k}.
\]
Let $\tilde{w}_i = \sqrt[H]{\frac{\rho\rho^{'}}{\mathcal{C}^{'}}}w_i$ for all $i = \{1,2,\dots,k\}$. Note that $\tilde{w}_i = \frac{1}{\sqrt{C}}w_i$. Thus
\begin{equation}
\mathcal{L}_{\Lambda,H}({\bm w}) = S\rho^{'} + q\mathcal{C}^{'}\sum_{i_1,\dots,i_H=1}^{\Lambda}X_{i_1,\dots,i_H}\prod_{k = 1}^{H}\tilde{w}_{i_k}.
\label{eq:befloss}
\end{equation}
Note that the spherical assumption in Equation~\ref{eq:befspherical} directly implies that
\begin{equation*}
\frac{1}{\Lambda}\sum_{i=1}^{\Lambda}\tilde{w}_i^2 = 1
\end{equation*}
To simplify the notation in Equation~\ref{eq:befloss} we drop the letter accents and simply denote $\tilde{w}$ as $w$. We skip constant $S\rho^{'}$ and $\mathcal{C}^{'}$ as it does not matter when minimizing the loss function. After substituting $q = \frac{1}{\Psi^{(H-1)/2H}}$ we obtain
\begin{equation*}
\mathcal{L}_{\!\Lambda,H}({\bm w}) \!=\! \frac{1}{\Lambda^{\!(H\!-\!1)/2}}\!\!\!\!\!\!\!\sum_{i_1,i_2,\dots,i_H=1}^{\Lambda}\!\!\!\!\!\!\!\!\!\!\!X_{i_1,i_2,\dots,i_H}\!w_{i_1}\!w_{i_2}\!\!\dots\!w_{i_H}.
\end{equation*}

\section{Asymptotics of the mean number of critical points and local minima}

Below, we provide the asymptotics of the mean number of critical points (Theorem~\ref{thm:cpprecise}) and the mean number of local minima (Theorem~\ref{thm:lmspprecise}), which extend Theorem~\ref{thm:lmsp}. Those results are the consequences of Theorem 2.17. and Corollary 2.18.~\cite{AAC2010}.
\begin{theorem}
For $H \geq 3$, the following holds as $\Lambda \rightarrow \infty$:
\begin{itemize}
\item For $u < -E_{\infty}$\\
\begin{eqnarray*}
\mathbb{E}[\mathcal{C}_{\Lambda}(u)] \!\!\!&=&\!\!\! \frac{h(v)}{\sqrt{2H\pi}}\frac{\exp(I_1(v) - \frac{v}{2}I_1^{'}(v))}{-\Phi^{'}(v) + I_1^{'}(v)}\Lambda^{-\frac{1}{2}}\\
&\cdot& \exp\left(\Lambda\Theta_H(u)\right)(1 + o(1)),
\end{eqnarray*}
where $v = -u\sqrt{\frac{H}{2(H-1)}}$, $\Phi(v) = -\frac{H-2}{2H}v^2$,\\
$h(v) = \left|\frac{v-\sqrt{2}}{v + \sqrt{2}}\right|^{\frac{1}{4}} + \left|\frac{v+\sqrt{2}}{v - \sqrt{2}}\right|^{\frac{1}{4}}$\\
and $I_1(v) = \int_{\sqrt{2}}^{v}\sqrt{|x^2 - 2|}dx$.
\item For $u = -E_{\infty}$\\
\begin{eqnarray*}
\mathbb{E}[\mathcal{C}_{\Lambda}(u)] = \frac{2A(0)\sqrt{2H}}{3(H-2)}\Lambda^{-\frac{1}{3}}\\
\cdot\exp\left(\Lambda\Theta_H(u)\right)(1 + o(1)),
\end{eqnarray*}
where $A$ is the Airy function of first kind.
\item For $u \in (-E_{\infty},0)$\\
\begin{eqnarray*}
\mathbb{E}[\mathcal{C}_{\Lambda}(u)] = \frac{2\sqrt{2H(E_{\infty}^2-u^2)}}{(2-H)\pi u}\\
\cdot\exp\left(\Lambda \Theta_H(u)\right)(1 + o(1)),
\end{eqnarray*}
\item For $u > 0$\\
\begin{eqnarray*}
\mathbb{E}[\mathcal{C}_{\Lambda}(u)] = \frac{4\sqrt{2}}{\sqrt{\pi(H-2)}}\Lambda^{\frac{1}{2}}\\
\cdot\exp\left(\Lambda\Theta_H(0)\right)(1 + o(1)),
\end{eqnarray*}
\end{itemize}
\label{thm:cpprecise}
\end{theorem}
\begin{theorem}
For $H \geq 3$ and $u < -E_{\infty}$, the following holds as $\Lambda \rightarrow \infty$:
\begin{eqnarray*}
\mathbb{E}[\mathcal{C}_{\Lambda,0}(u)] \!\!\!&=&\!\!\! \frac{h(v)}{\sqrt{2H\pi}}\frac{\exp(I_1(v) - \frac{v}{2}I_1^{'}(v))}{-\Phi^{'}(v) + I_1^{'}(v)}\Lambda^{-\frac{1}{2}}\\
&\cdot& \exp\left(\Lambda\Theta_H(u)\right)(1 + o(1)),
\end{eqnarray*}
where $v$, $\Phi$, $h$ and $I_1$ were defined in Theorem~\ref{thm:cpprecise}.
\label{thm:lmspprecise}
\end{theorem}

\section{Additional Experiments}

\subsection{Distribution of normalized indices of critical points.}

Figure~\ref{fig:index_distrall} shows the distribution of normalized indices, which is the proportion of negative eigenvalues, for neural networks with $n_1=\{10,25,50,100\}$.
We see that all solutions are minima or saddle points of very low index.
\begin{figure}[h]
\vspace{-0.1in}
a) $n_1=10$\\ \includegraphics[trim=0cm 0cm 0cm 2.25cm,clip,scale=0.4]{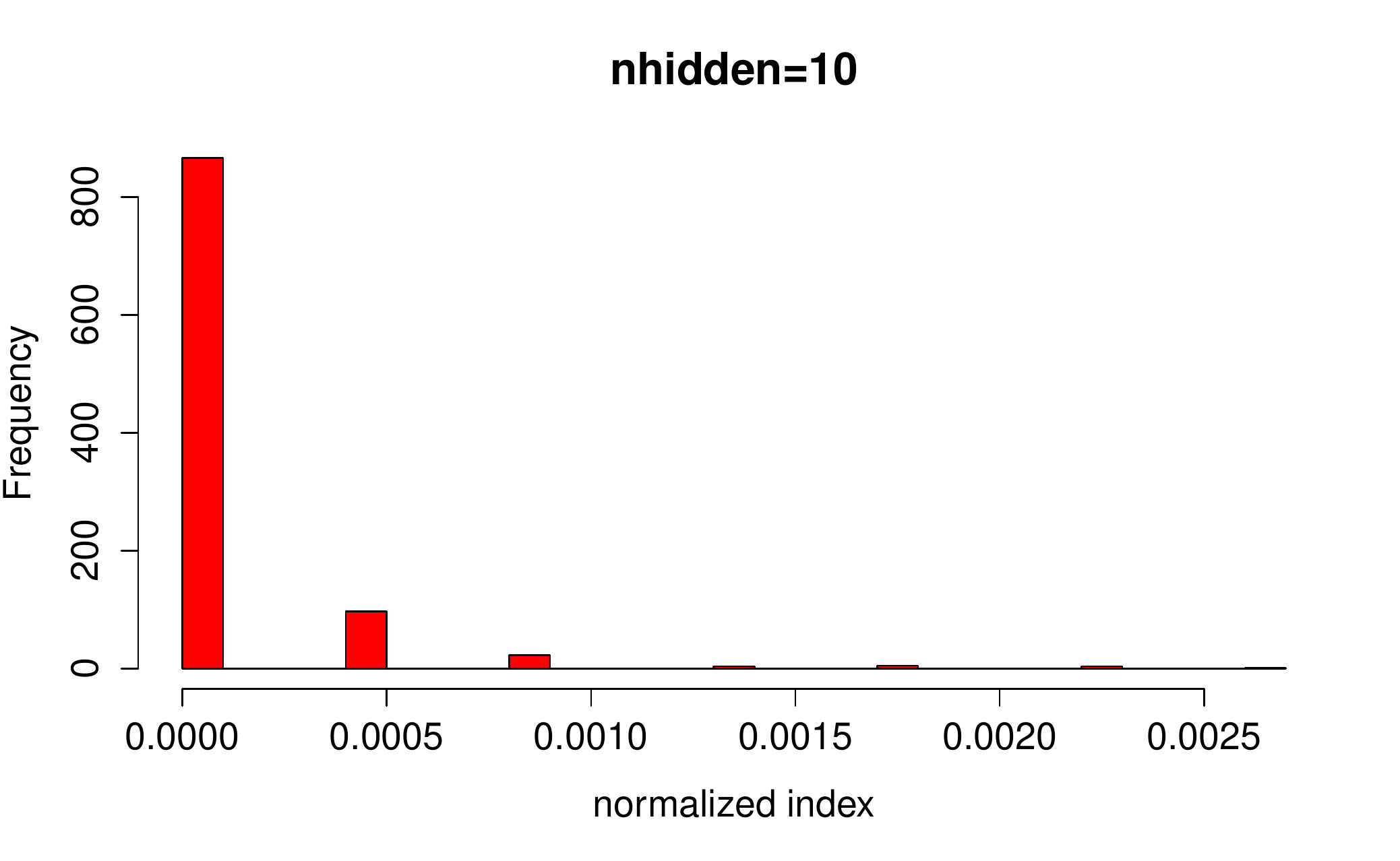}\vspace{-0.1in}
b) $n_1=25$\\ \includegraphics[trim=0cm 0cm 0cm 2.25cm,clip,scale=0.4]{hist_index_h25.pdf}\vspace{-0.1in}
c) $n_1=50$\\ \includegraphics[trim=0cm 0cm 0cm 2.25cm,clip,scale=0.4]{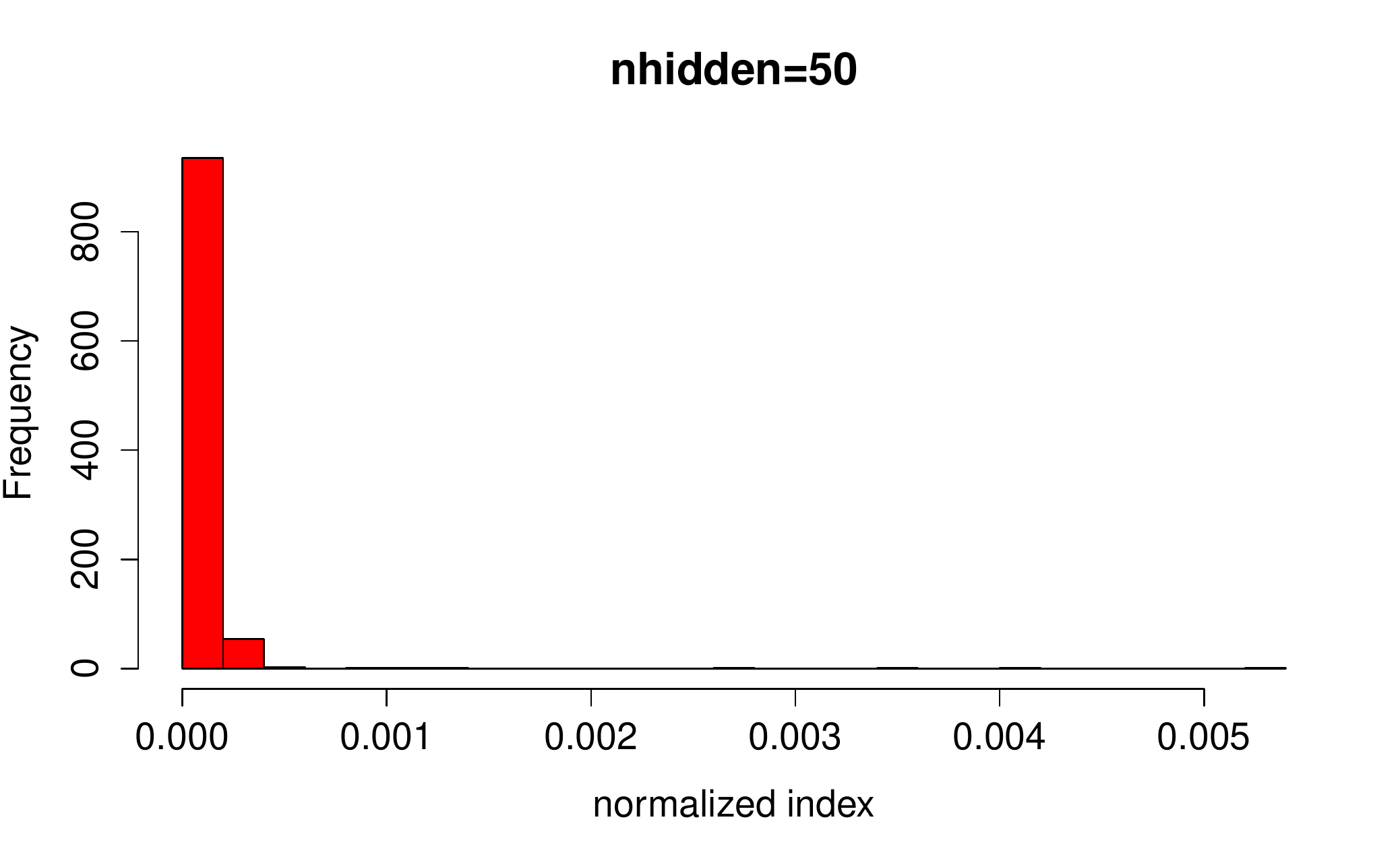}\vspace{-0.1in}
d) $n_1=100$\\ \includegraphics[trim=0cm 0cm 0cm 1.80cm,clip,scale=0.4]{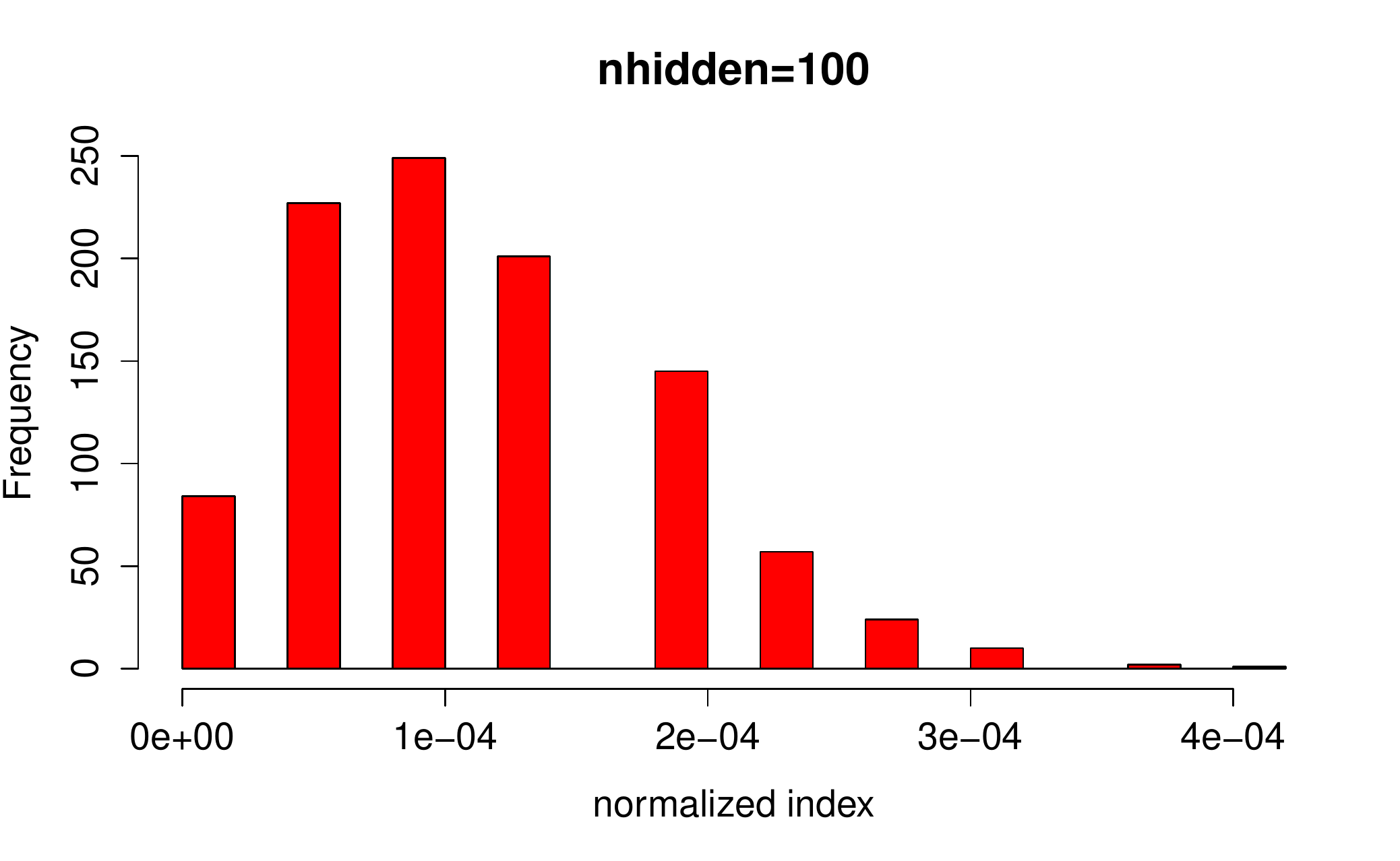}
\vspace{-0.3in}
\caption{Distribution of normalized index of solutions for $n_1=\{10,25,50,100\}$ hidden units.}
\label{fig:index_distrall}
\end{figure}

\newpage
\subsection{Comparison of SGD and SA.}

Figure \ref{fig:sa_vs_sgd} compares SGD with SA.
\begin{figure}[h]
\includegraphics[scale=0.4]{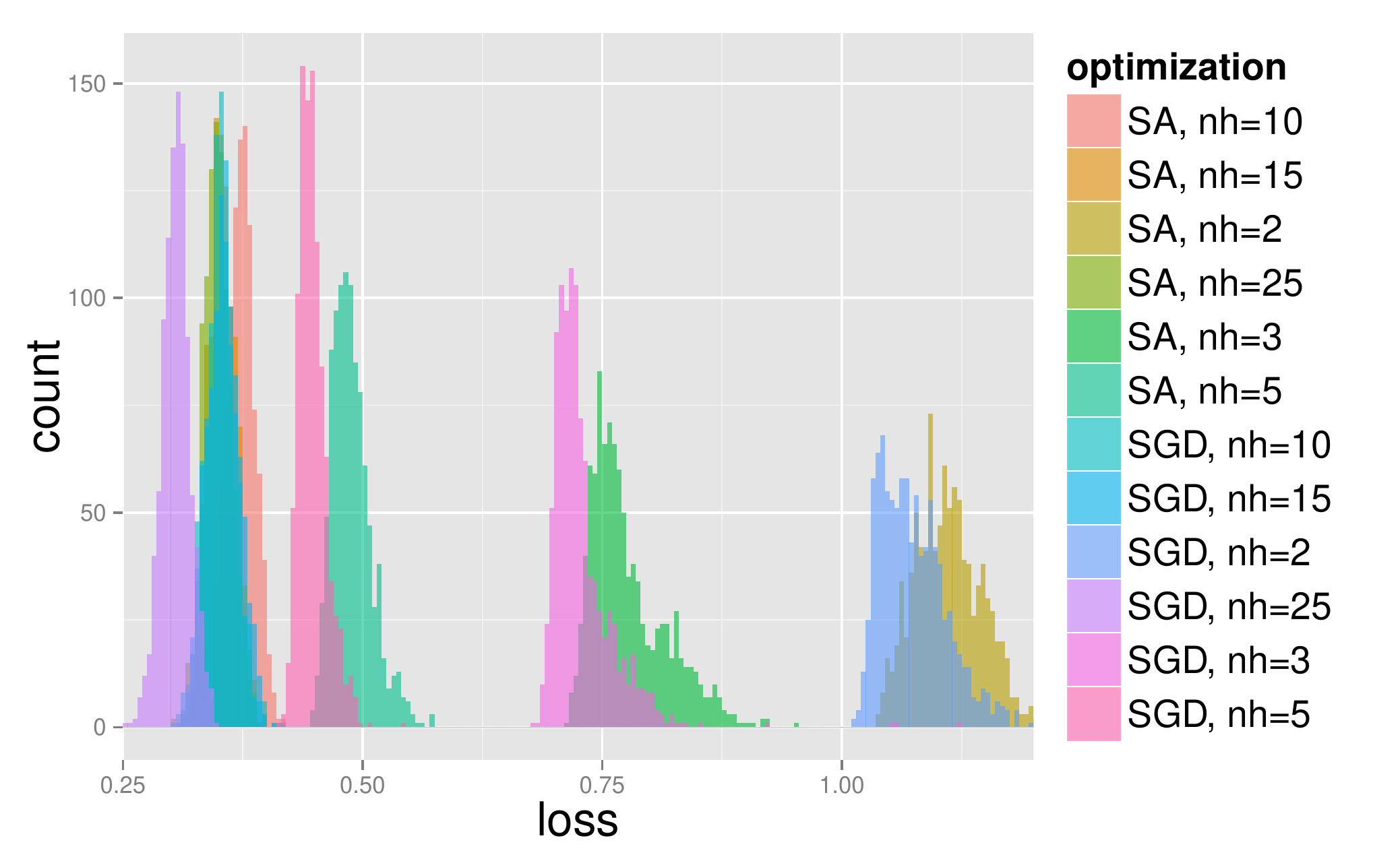}
\vspace{-0.27in}
\caption{Test loss distributions for SGD and SA for different numbers of hidden units (nh).}
 \label{fig:sa_vs_sgd}
\vspace{-0.1in}
\end{figure}

\begin{figure*}
  \center
\hspace{-0.8in}a) $n_1=2\:\:\:\:\:\:\:\:\:\:\:\:\:\:\:\:\:\:\:\:\:\:\:\:\:\:\:\:\:\:\:\:\:\:\:\:$b) $n_1=5\:\:\:\:\:\:\:\:\:\:\:\:\:\:\:\:\:\:\:\:\:\:\:\:\:\:\:\:\:\:\:\:\:\:\:\:\:$c) $n_1=10\:\:\:\:\:\:\:\:\:\:\:\:\:\:\:\:\:\:\:\:\:\:\:\:\:\:\:\:\:\:\:\:\:\:\:\:$d) $n_1=25$\\
\includegraphics[trim=0cm 0cm 0cm 1.4cm,clip,width = 1.69in]{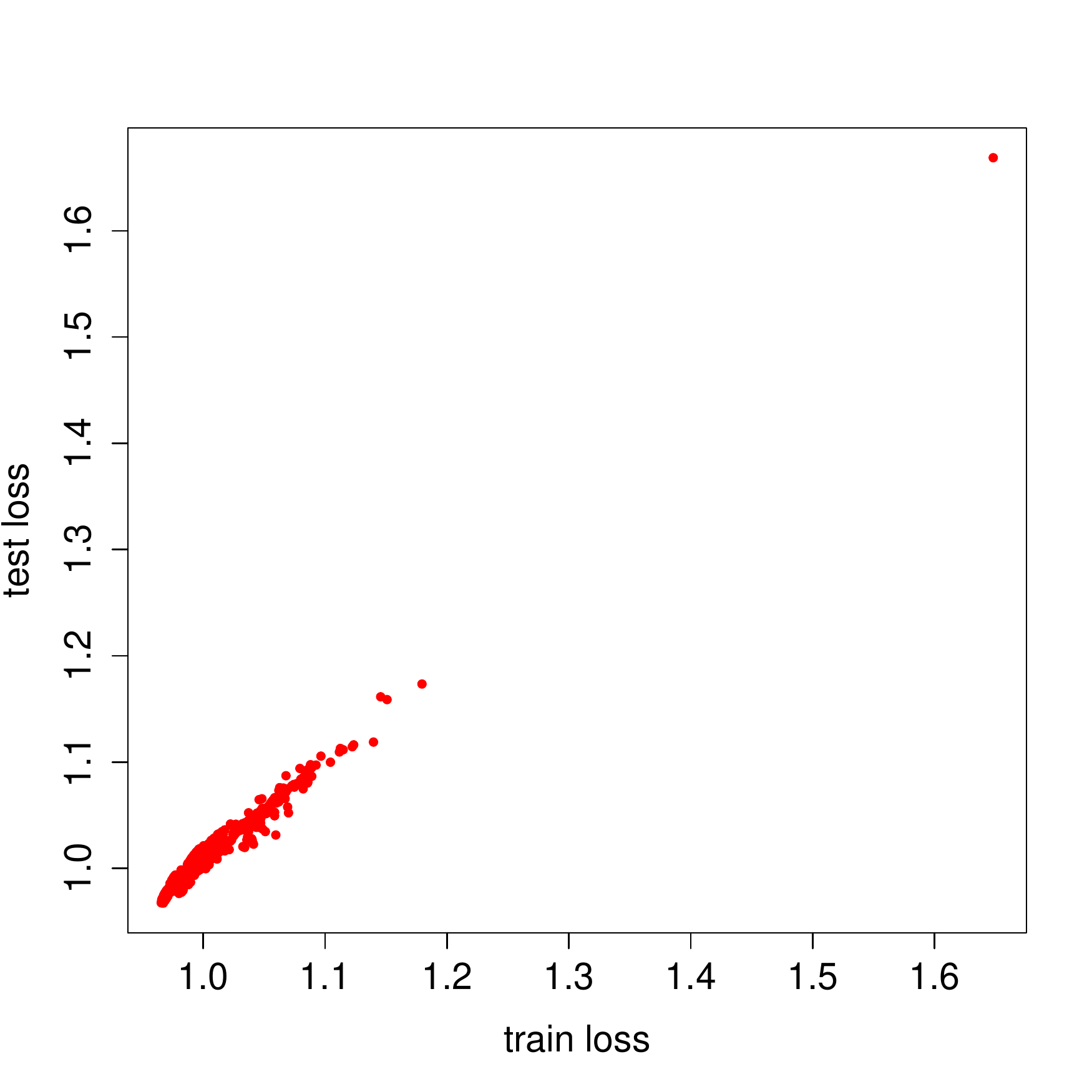}
\hspace{-0.05in}\includegraphics[trim=0cm 0cm 0cm 1.4cm,clip,width = 1.69in]{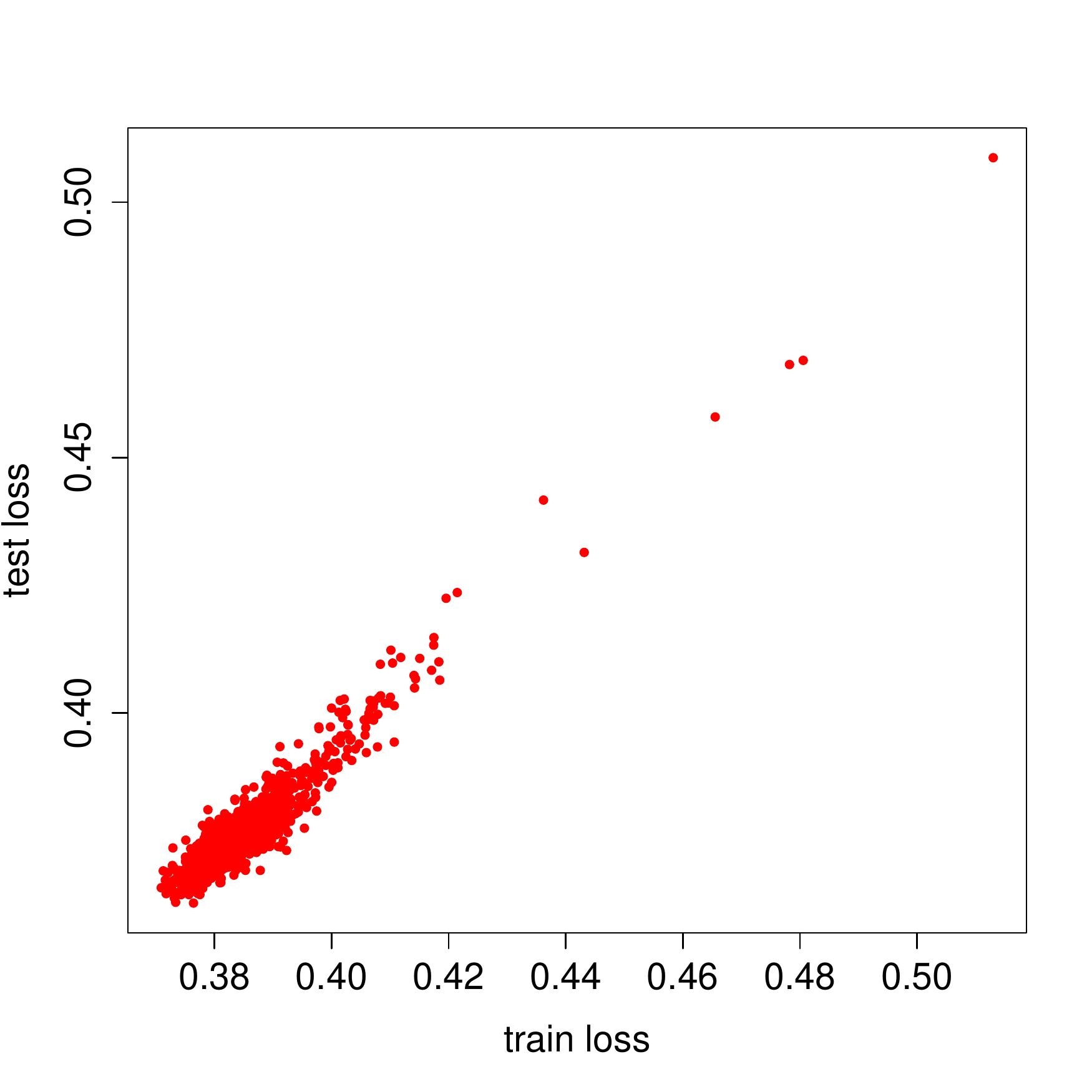} 
\hspace{-0.05in}\includegraphics[trim=0cm 0cm 0cm 1.4cm,clip,width = 1.69in]{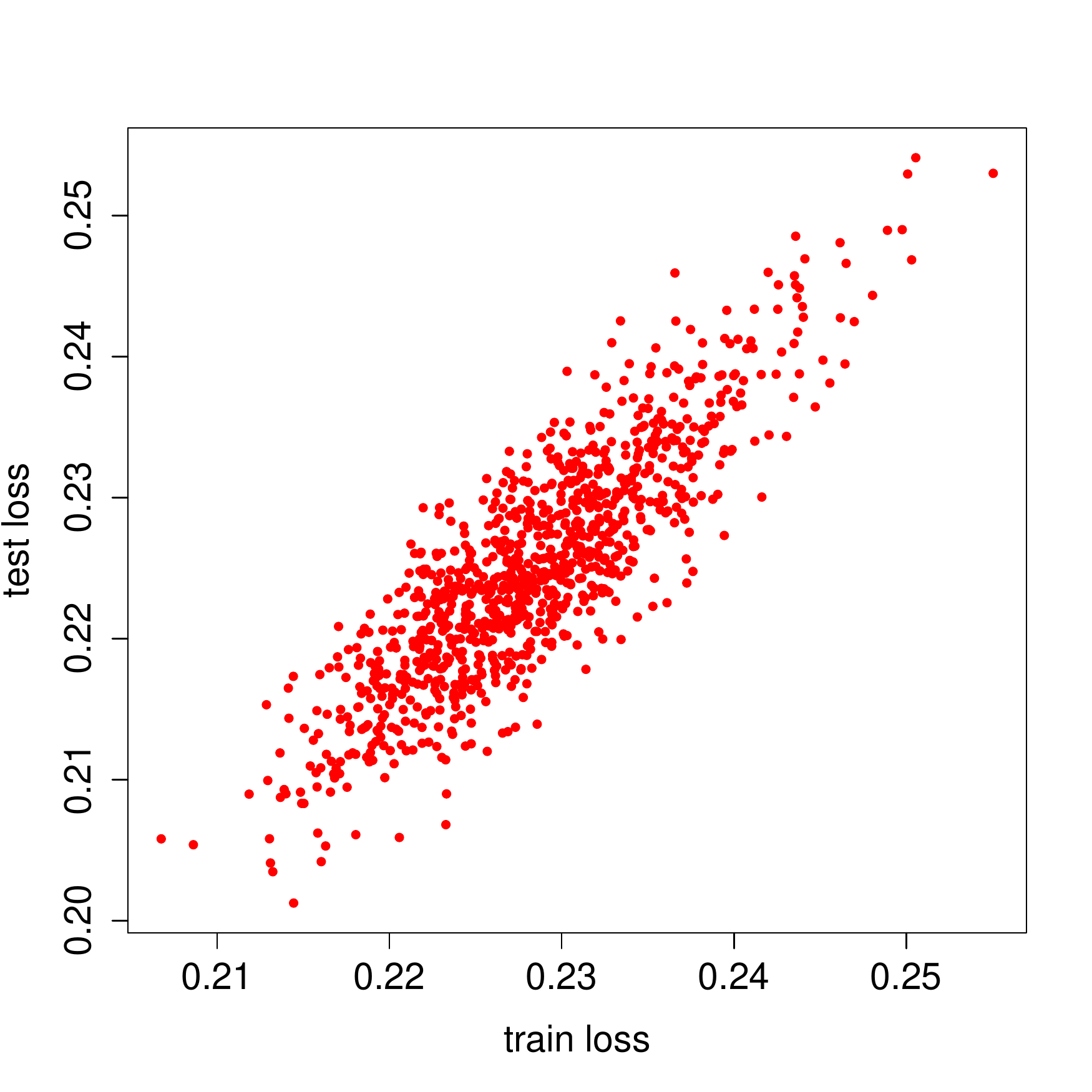} 
\hspace{-0.05in}\includegraphics[trim=0cm 0cm 0cm 1.4cm,clip,width = 1.69in]{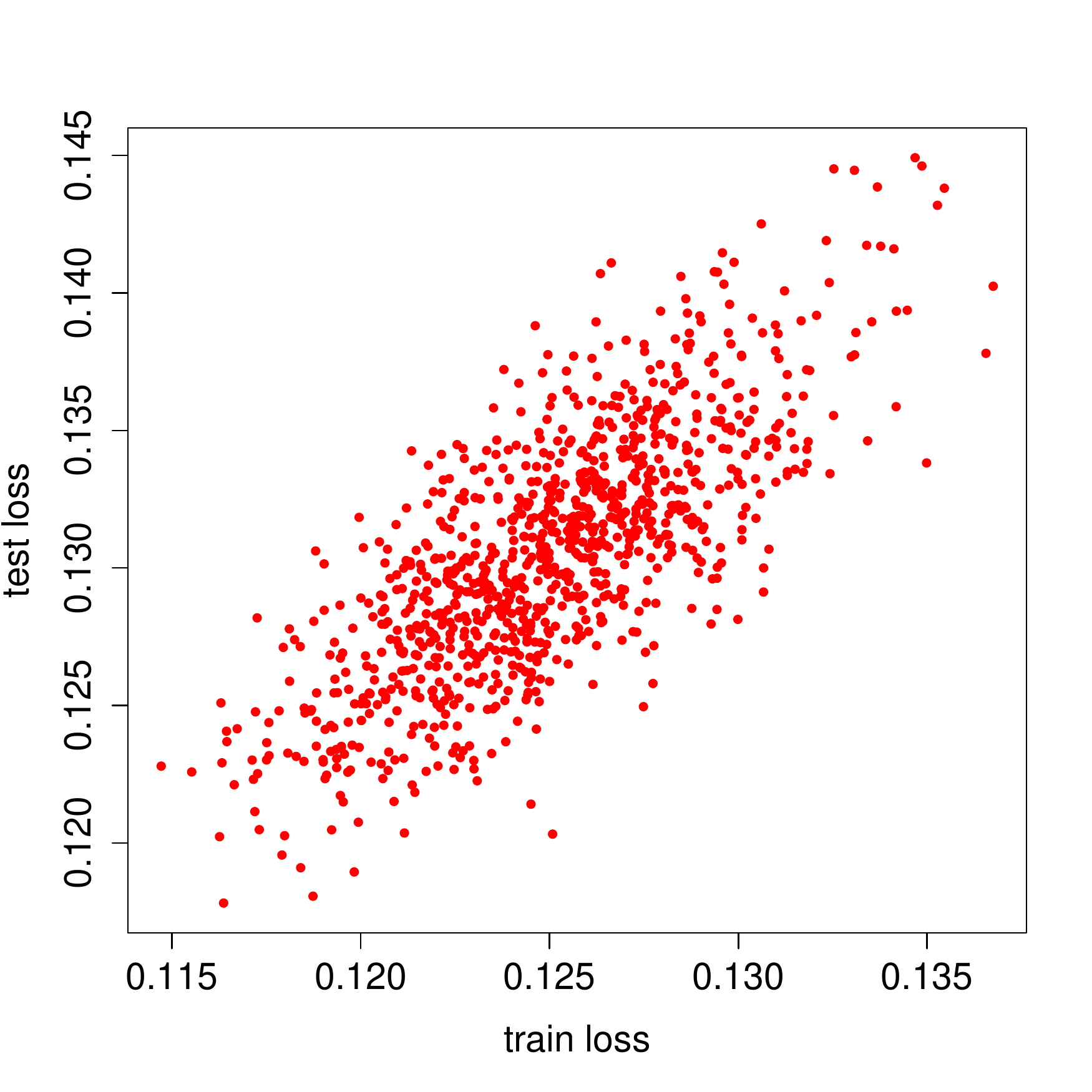}\\
\hspace{-0.7in}e) $n_1=50\:\:\:\:\:\:\:\:\:\:\:\:\:\:\:\:\:\:\:\:\:\:\:\:\:\:\:\:\:\:\:\:\:\:$f) $n_1=100\:\:\:\:\:\:\:\:\:\:\:\:\:\:\:\:\:\:\:\:\:\:\:\:\:\:\:\:\:\:\:\:\:$g) $n_1=250\:\:\:\:\:\:\:\:\:\:\:\:\:\:\:\:\:\:\:\:\:\:\:\:\:\:\:\:\:\:\:\:$h) $n_1=500$\\
\includegraphics[trim=0cm 0cm 0cm 1.4cm,clip,width = 1.69in]{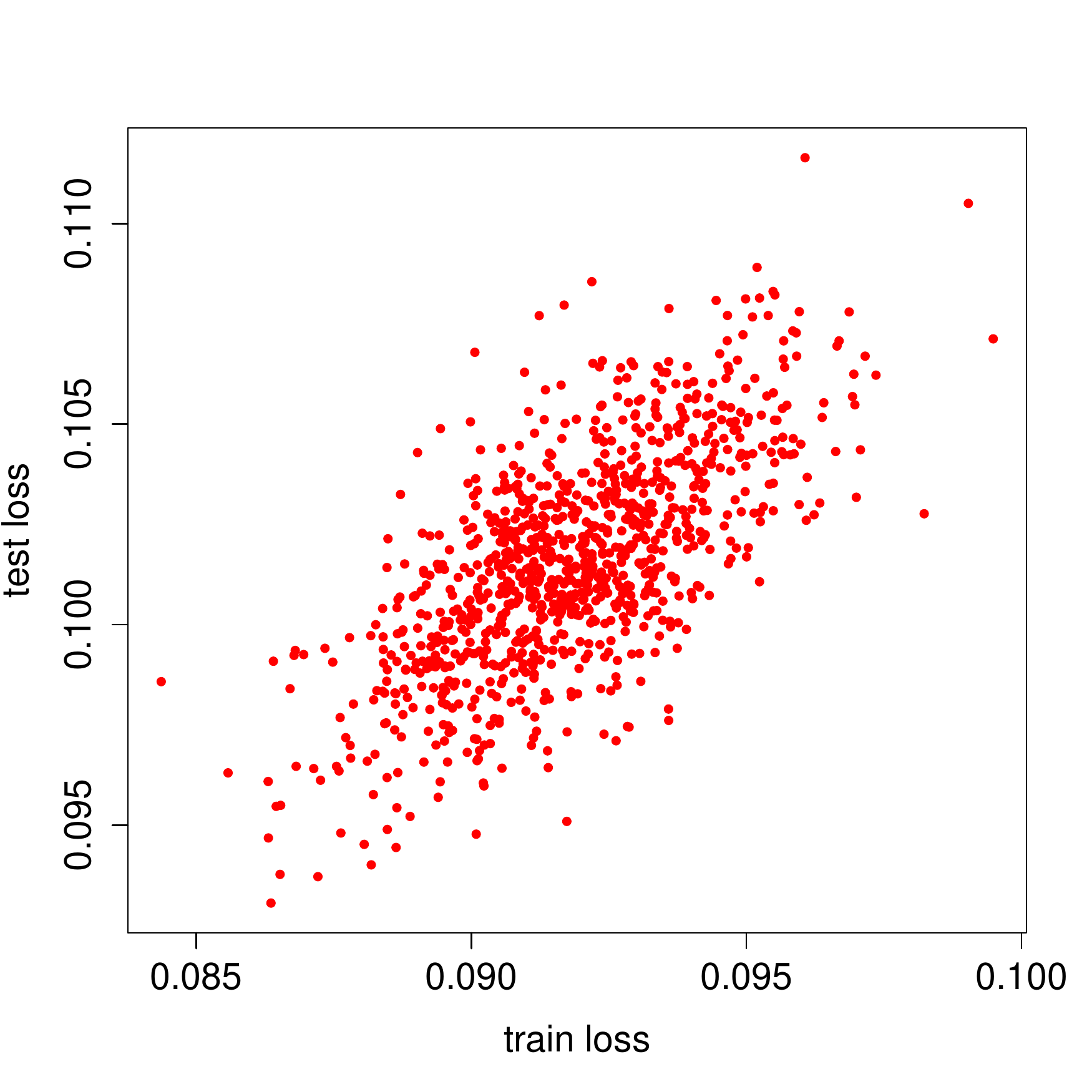}
\hspace{-0.05in}\includegraphics[trim=0cm 0cm 0cm 1.4cm,clip,width = 1.69in]{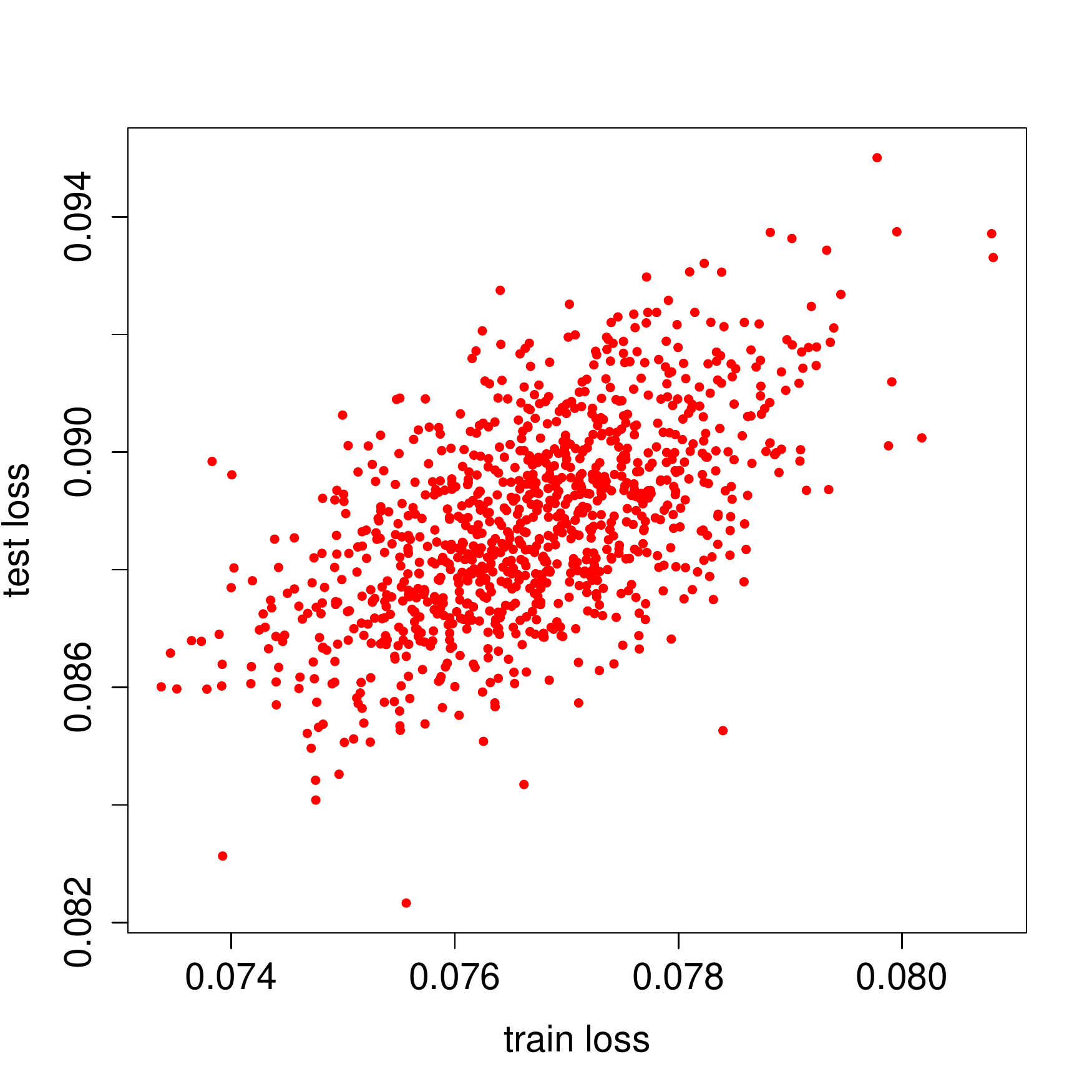}
\hspace{-0.05in}\includegraphics[trim=0cm 0cm 0cm 1.4cm,clip,width = 1.69in]{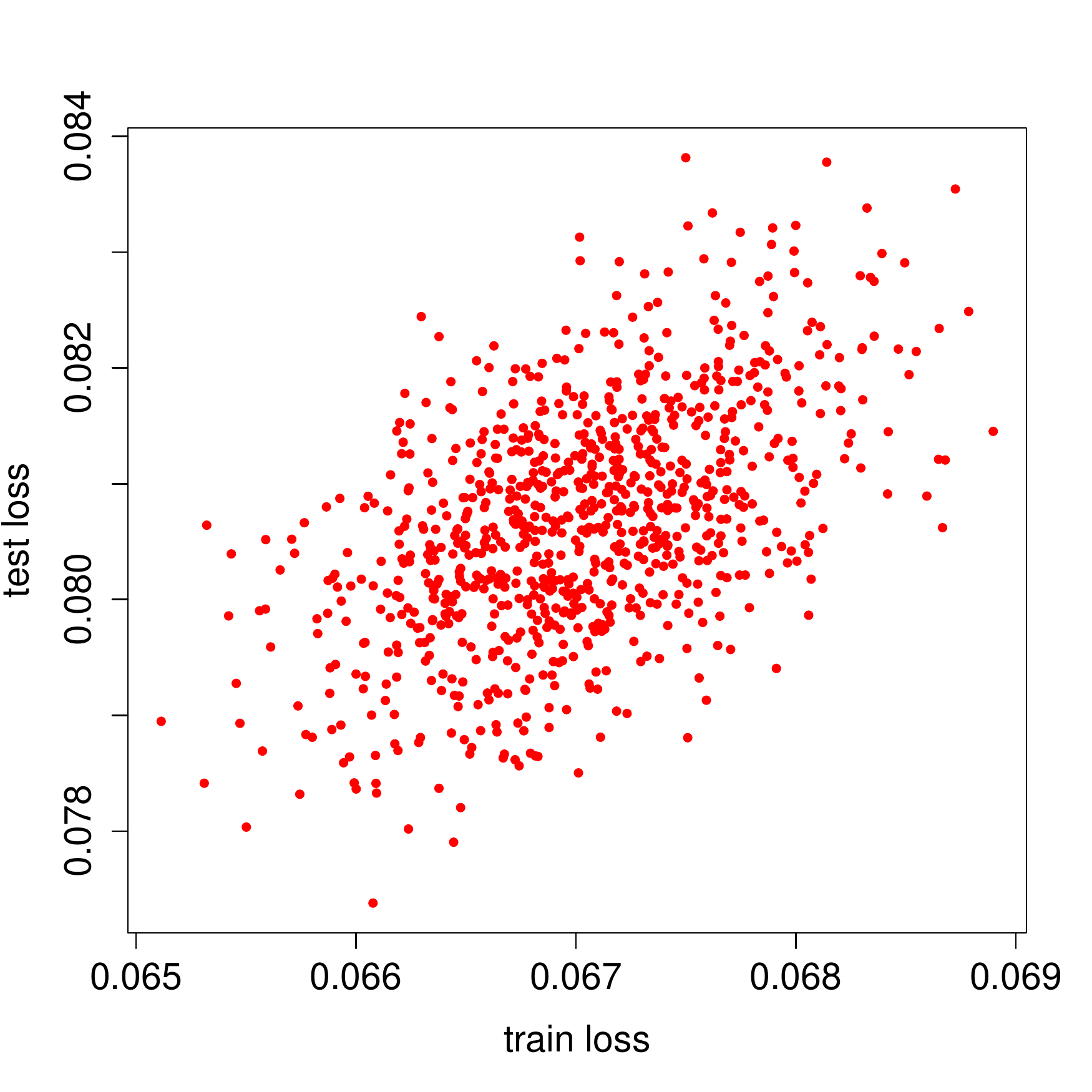}
\hspace{-0.05in}\includegraphics[trim=0cm 0cm 0cm 1.4cm,clip,width = 1.69in]{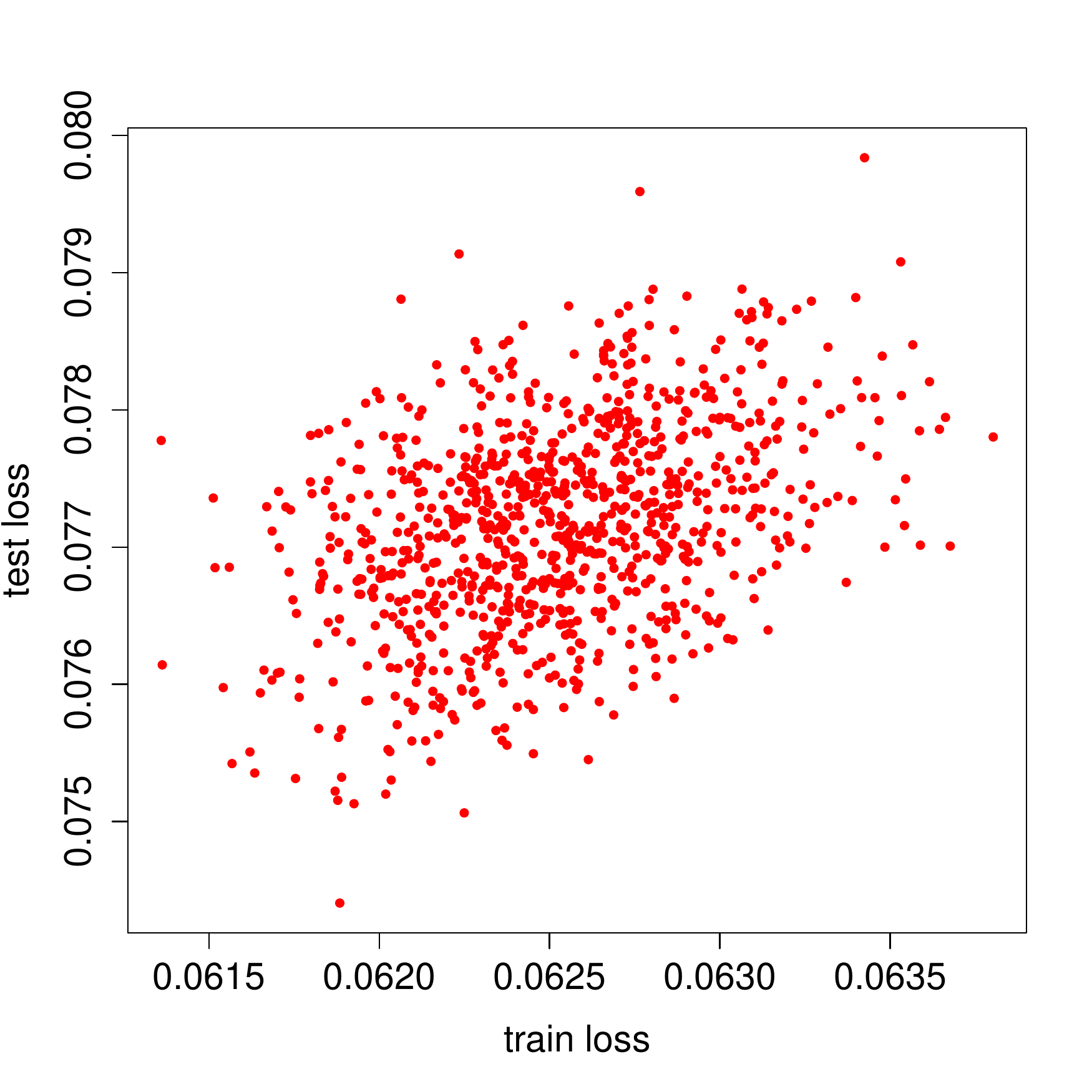}     
\vspace{-0.25in}
\caption{Test loss versus train loss for networks with different number of hidden units $n_1$.}
\label{fig:correlations}
\end{figure*}

\subsection{Distributions of the scaled test losses}

\begin{figure}[h]
  \center
\includegraphics[width = 3.0in]{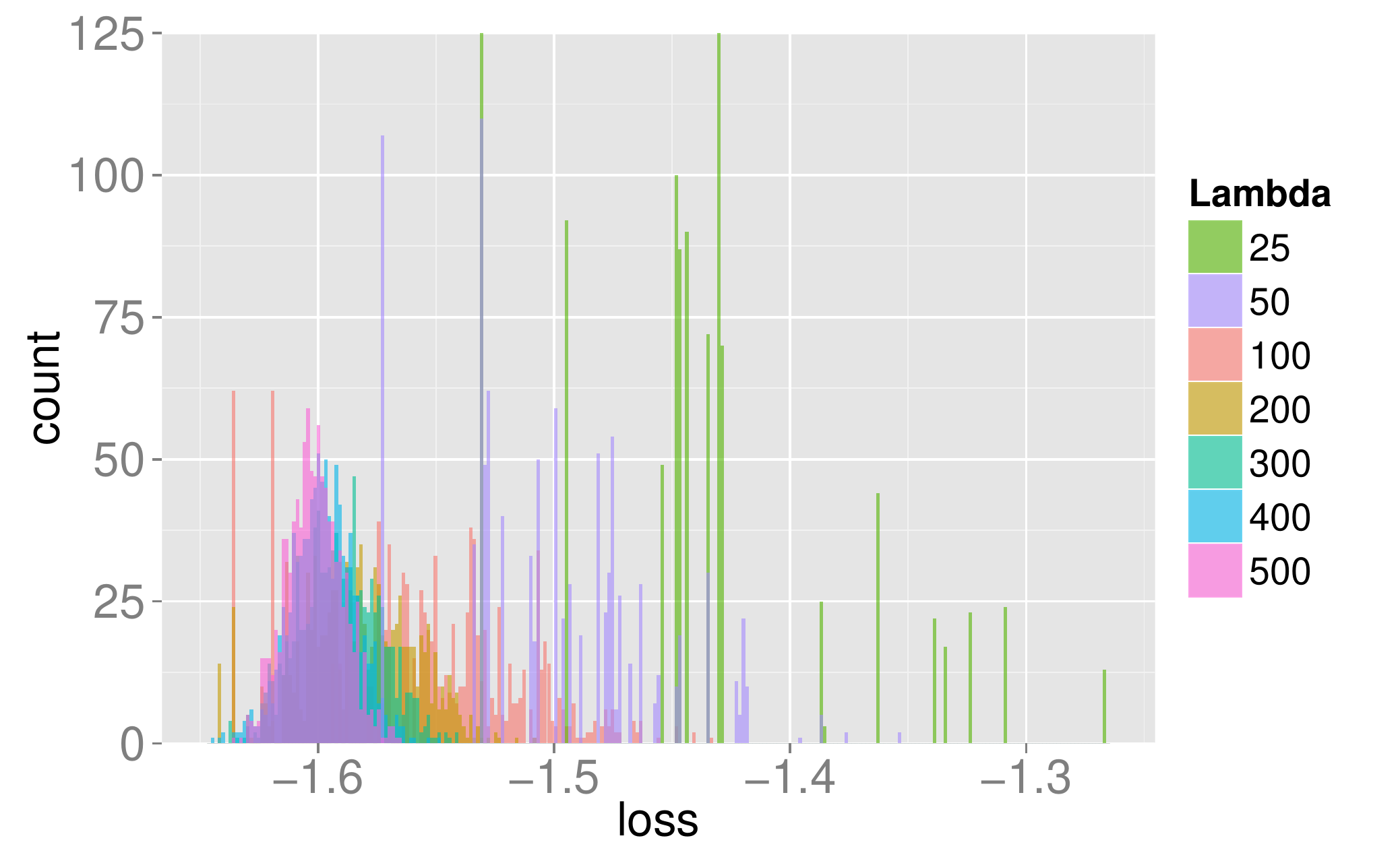}
\includegraphics[width = 3.0in]{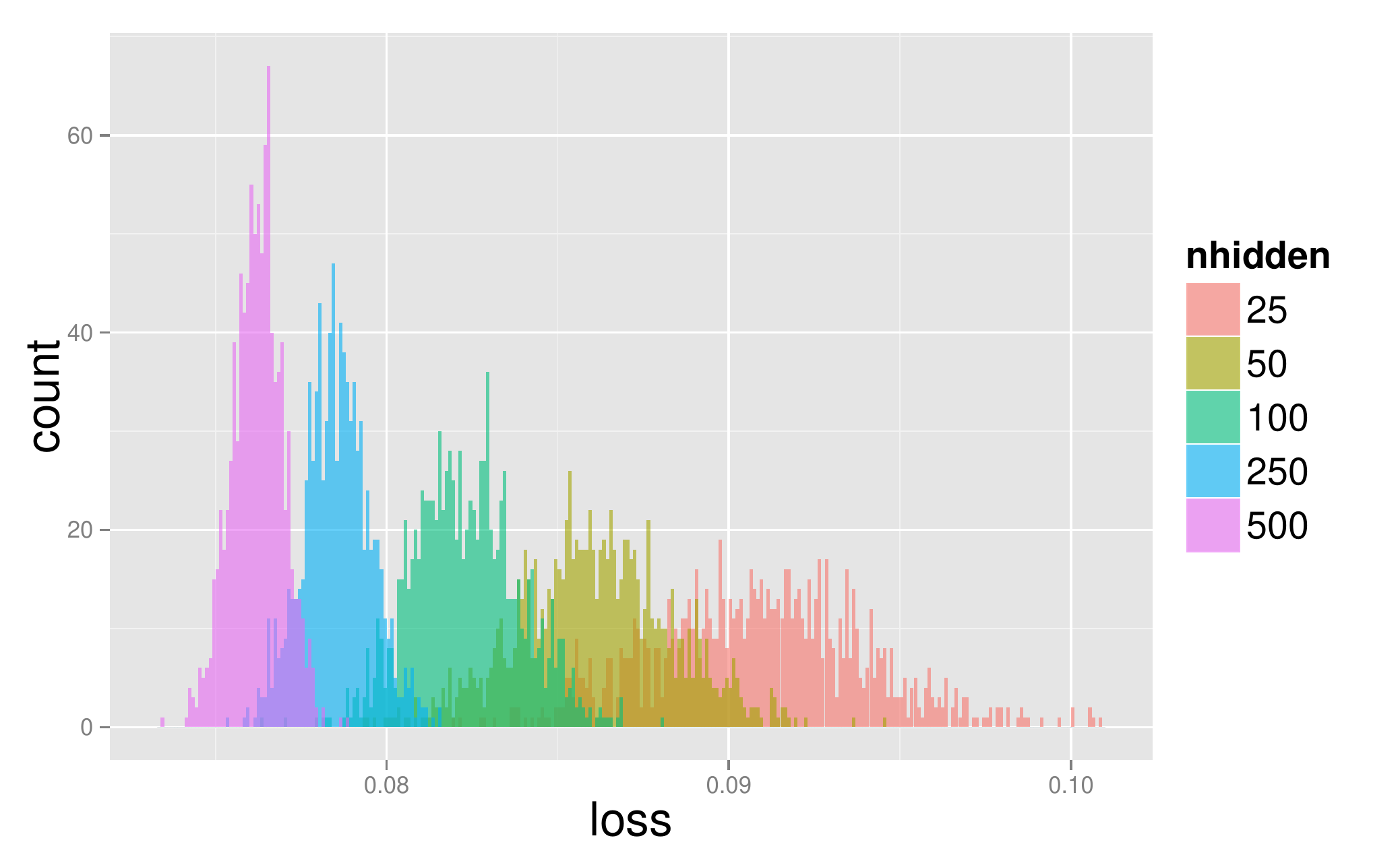} 
\vspace{-0.15in}
\caption{Distributions of the scaled test losses for the spin-glass (with $\Lambda = \{25,50,100,200,300,400,500\}$) (top) and the neural network (with $n_1 = \{25,50,100,250,500\}$) (bottom) experiments.}
\label{fig:spinglass_and_mnist2}
\end{figure}

Figure~\ref{fig:spinglass_and_mnist2} shows the distributions of the scaled test losses for the spin-glass experiment  (with $\Lambda = \{25,50,100,200,300,400,500\}$) and the neural network experiment (with $n_1 = \{25,50,100,250,500\}$). Figure~\ref{fig:mnist_statistics1} captures the boxplot generated based on the distributions of the scaled test losses for the neural network experiment (for $n_1 =\{10,25,50,100,250,500\}$) and its zoomed version (for $n_1 =\{10,25,50,100\}$).

\begin{figure}[htp!]
  \center
\includegraphics[trim=0cm 0cm 0cm 1.4cm,clip,scale=0.3]{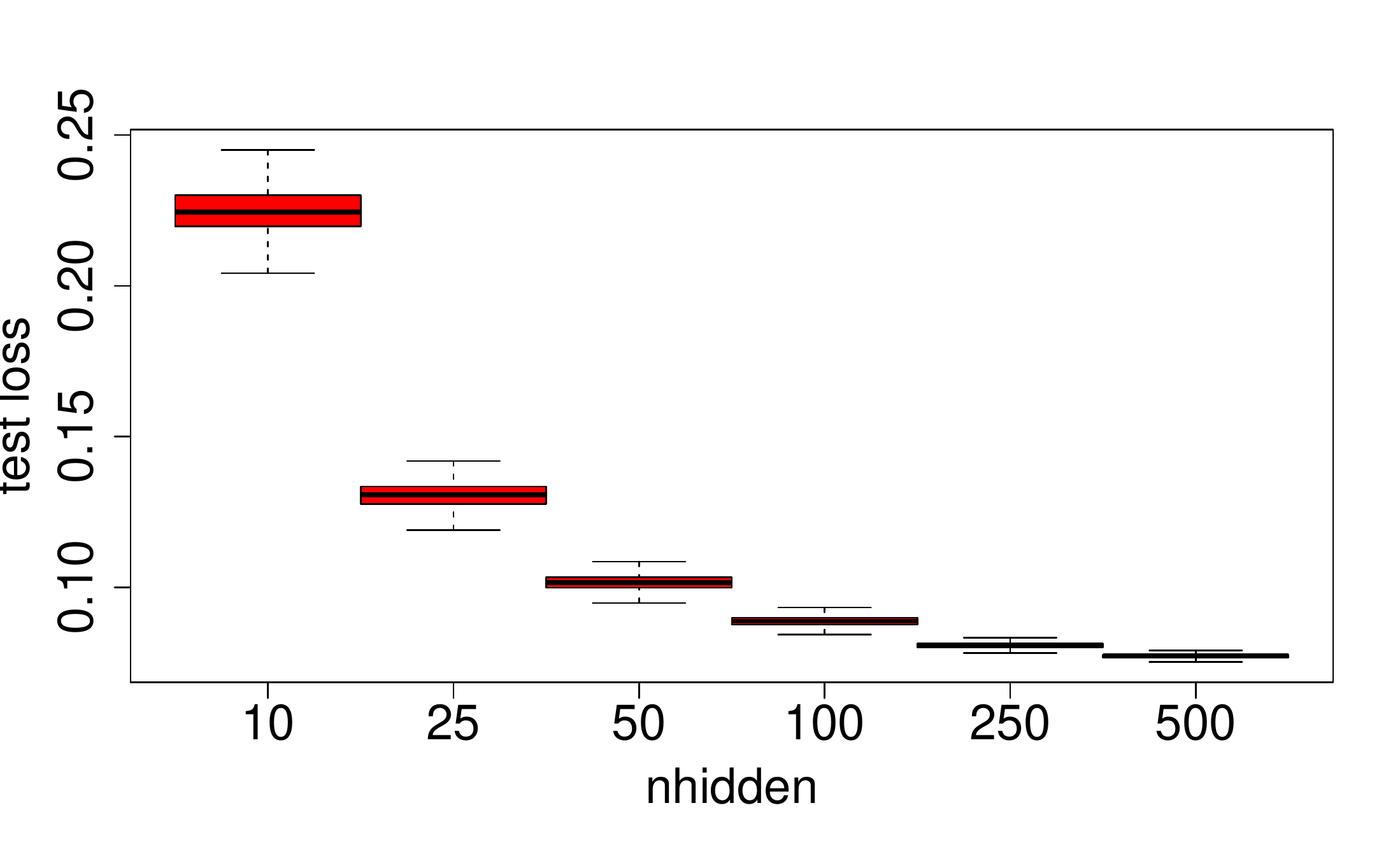}\\
\includegraphics[trim=0cm 0cm 0cm 1.4cm,clip,scale=0.3]{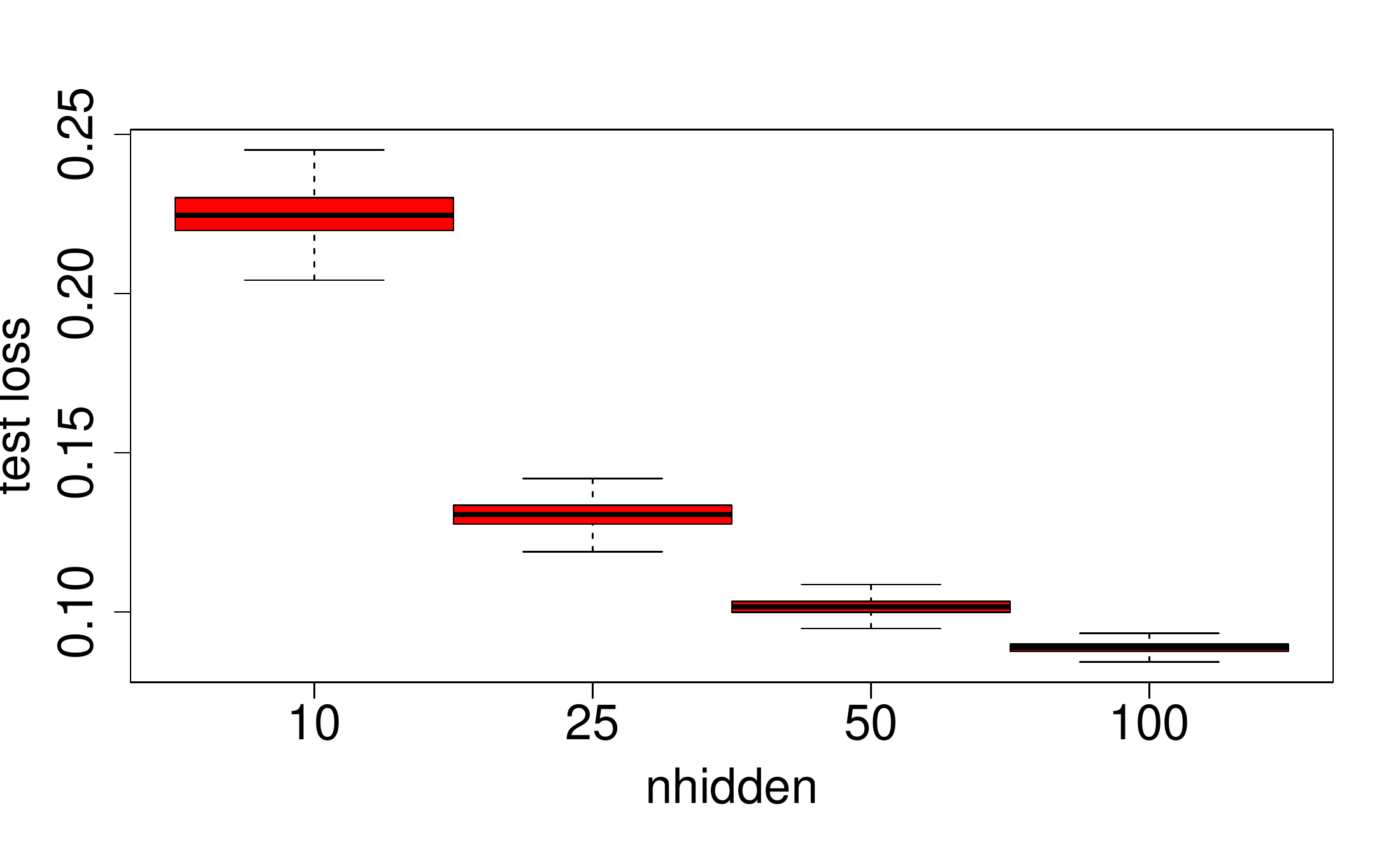}
\vspace{-0.15in}
\caption{\textbf{Top:} Boxplot generated based on the distributions of the scaled test losses for the neural network experiment, \textbf{Bottom:} Zoomed version of the same boxplot for $n_1 =\{10,25,50,100\}$.}
\label{fig:mnist_statistics1}
\end{figure}

Figure~\ref{fig:meanvar}  shows the mean value and the variance of the test loss as a function of the number of hidden units.
\begin{figure}[h]
  \center
\hspace{-0.05in}\includegraphics[trim=0cm 0cm 0cm 1.4cm,clip,scale=0.3]{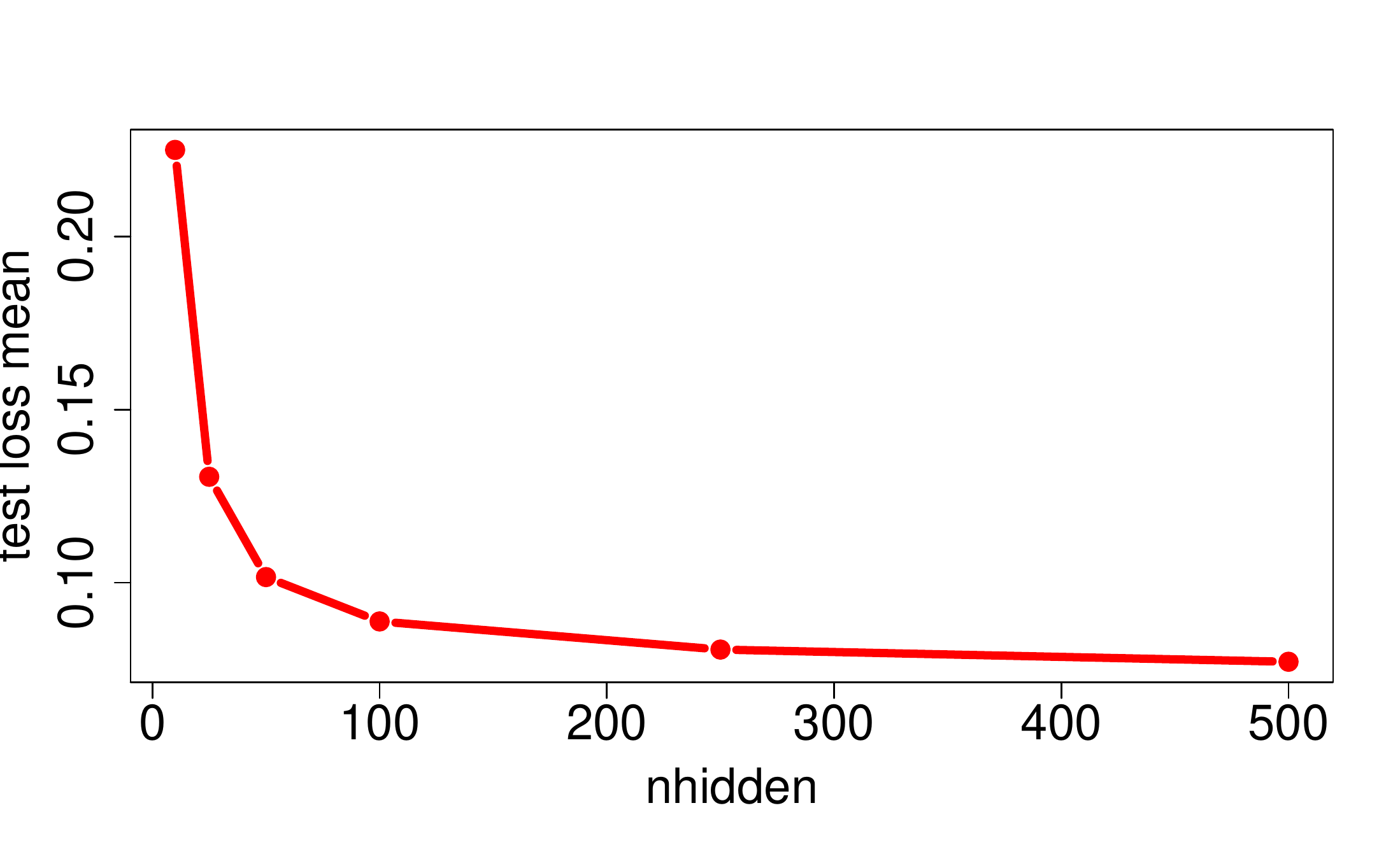}\\
\hspace{-0.05in}\includegraphics[trim=0cm 0cm 0cm 1.4cm,clip,scale=0.3]{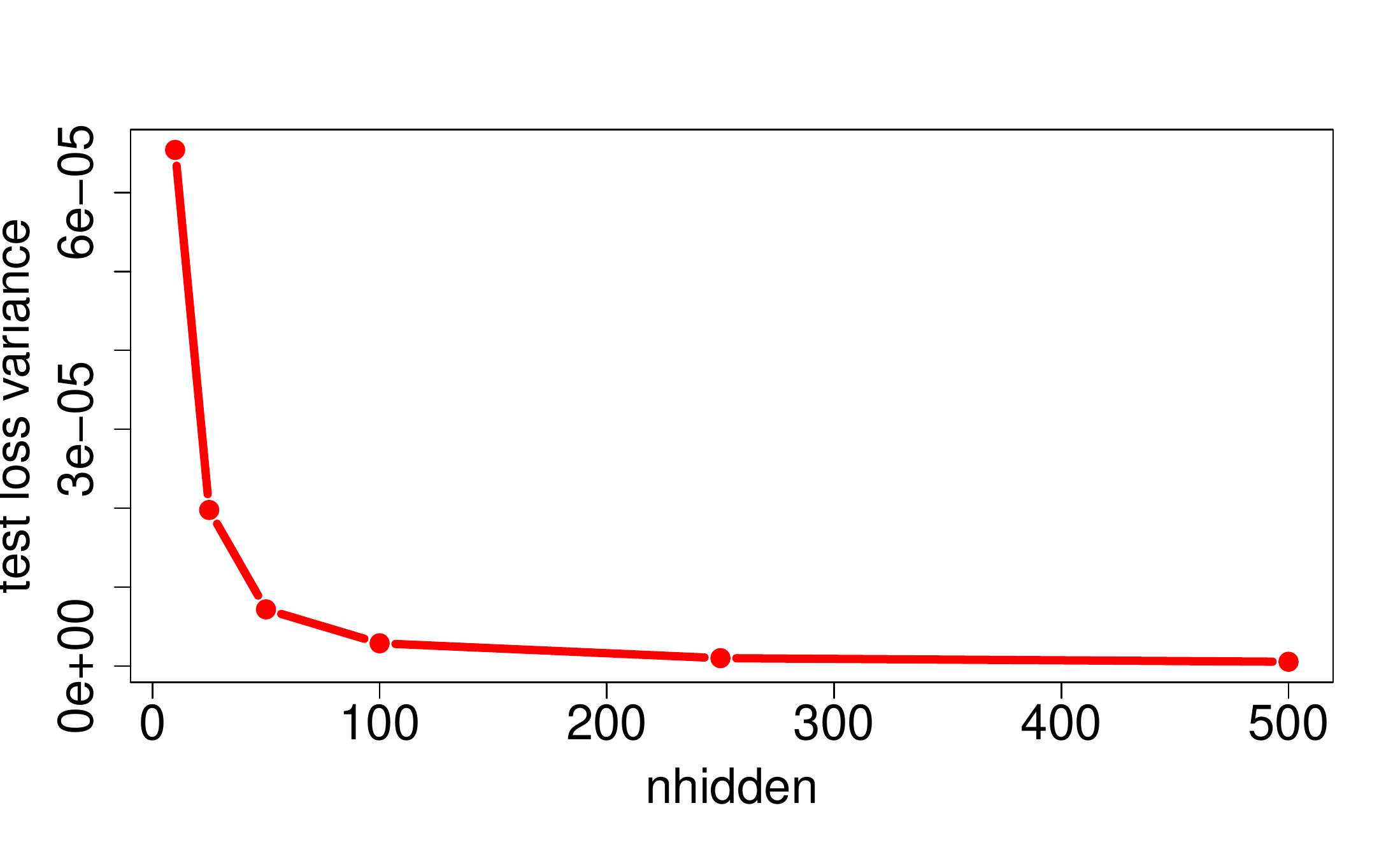} 
\vspace{-0.15in}
\caption{Mean value and the variance of the test loss as a function of the number of hidden units.}
\label{fig:meanvar}
\vspace{-0.1in}
\end{figure}

\newpage
\subsection{Correlation between train and test loss}

Figure~\ref{fig:correlations} captures the correlation between training and test loss for networks with different number of hidden units $n_1$.

\end{document}